\documentclass[12pt]{article}

\usepackage[utf8]{inputenc}
\usepackage[T1]{fontenc}
\usepackage{lmodern}
\usepackage[margin=1in]{geometry}
\usepackage{amsmath,amssymb,amsthm}
\usepackage{mathtools}
\usepackage{booktabs}
\usepackage{array}
\usepackage[colorlinks=true,linkcolor=black,citecolor=black,urlcolor=blue]{hyperref}
\usepackage{cleveref}
\usepackage{natbib}
\usepackage{graphicx}
\usepackage{xcolor}
\usepackage{enumitem}
\usepackage{placeins}
\usepackage{float}
\usepackage{multirow}

\newtheorem{theorem}{Theorem}
\newtheorem{proposition}{Proposition}
\newtheorem{lemma}{Lemma}
\newtheorem{corollary}{Corollary}
\newtheorem{fact}{Fact}
\newtheorem{conjecture}{Conjecture}
\theoremstyle{definition}
\newtheorem{definition}{Definition}
\theoremstyle{remark}
\newtheorem{remark}{Remark}

\newcommand{\R}{\mathbb{R}}
\newcommand{\E}{\mathbb{E}}
\newcommand{\Var}{\operatorname{Var}}
\newcommand{\Cov}{\operatorname{Cov}}
\newcommand{\Corr}{\operatorname{Corr}}
\newcommand{\CV}{\operatorname{CV}}
\newcommand{\tr}{\operatorname{tr}}
\newcommand{\sign}{\operatorname{sign}}
\newcommand{\ReLU}{\operatorname{ReLU}}
\newcommand{\RMSNorm}{\operatorname{RMSNorm}}
\newcommand{\diag}{\operatorname{diag}}
\newcommand{\Unif}{\operatorname{Unif}}
\newcommand{\Dir}{\operatorname{Dir}}
\newcommand{\xhat}{\hat{x}}
\newcommand{\yhat}{\hat{y}}

\title{\textbf{A Geometric Analysis of Sign-Magnitude Asymmetry\\
in a ReLU + RMSNorm Block under Ternary Quantization}}
\author{Dong Lei \\
\textit{Independent Researcher} \\
\texttt{dongleiit@hotmail.com}}
\date{}

\begin{document}
\maketitle

\begin{abstract}
Pre-norm Transformers with RMSNorm tolerate ternary $\{-1,0,+1\}$ weight
quantization with surprisingly small loss \citep{ma2024bitnet}. We give a
geometric explanation via sign-magnitude decomposition of weight perturbations.
In a two-layer ReLU + RMSNorm model with i.i.d.\ Gaussian weights, sign-flips
produce $\pi/(\pi-2) \approx 2.75$ times more transverse output energy than
sign-preserving magnitude perturbations of equal Frobenius norm, as the flip
rate $p \to 0$ (Theorem~3). The mechanism: ReLU creates a hidden-space
directional asymmetry between the two perturbation types, which RMSNorm's
transverse-projection Fr\'{e}chet derivative selectively exposes.
Sign-quantization error is itself a sign-preserving perturbation with angular
alignment $\cos^2 \to 2/\pi$ (Theorem~4); its post-ReLU radial fraction
($0.365$) matches the pre-ReLU value $1-2/\pi$ within $0.4\%$, so ReLU is
approximately transparent to ternary error. Multi-layer compounding of the
$2.75\times$ factor is not experimentally supported; the gap to real-model sign
sensitivity arises from outlier features violating delocalization. For an
input dimension with amplitude $\alpha$, a single sign-flip produces post-ReLU
energy amplified by $R \approx n\alpha^2$ relative to a delocalized entry.
On TinyLlama-1.1B, at linear response ($p \leq 0.5\%$), count-matched NLL
leverage stabilizes at ${\sim}10\times \approx n\mathbb{E}[\alpha^2]$, matching
the per-entry theory; the all-column NLL ratio of $5.0\times$ falls within
$R_{\mathrm{col}} \leq 19$ ($67\times$ PPL gap reflects metric nonlinearity). Measured
outlier $\alpha$ at layer~12 (median $0.024$, max $0.26$) confirms heavy-tailed
concentration. The
Bussgang constant $2/\pi$, RMSNorm geometry, and ReLU half-space structure
together explain sign-magnitude asymmetry in pre-norm models, with
$R \propto n\alpha^2$ accounting for real-model deviations.
\end{abstract}

\section{Introduction}

\subsection{Problem Statement}

Recent work on extremely low-bit large language models---notably BitNet b1.58 \citep{ma2024bitnet}---demonstrates that pre-norm Transformer architectures \citep{xiong2020layer} can operate with weights constrained to the ternary set $\{-1, 0, +1\}$ while retaining most of their linguistic capability. This is remarkable: ternary quantization discards all magnitude information, preserving only the sign pattern and sparsity structure. Why should this suffice?

Empirically, the importance of weight signs over magnitudes has been documented since the Lottery Ticket Hypothesis era. \citet{zhou2019signs} showed that randomly re-initializing magnitudes while keeping signs fixed preserves subnetwork trainability; \citet{oh2025sign} extended this to full networks. \citet{vardhan2026} documented a \emph{crossover dissociation} in LayerNorm-based models (Pythia): angular perturbations dominate language-modeling damage (up to $42.9\times$), while magnitude perturbations dominate syntactic processing damage. Crucially, they note that this pattern may not generalize to other architectural families, where the direction-magnitude geometry differs. Yet a rigorous theoretical account of \emph{why} sign carries disproportionate functional information in pre-norm (RMSNorm) architectures has been lacking.

\subsection{Contributions}

This paper provides four categories of results:

\begin{enumerate}
\item \textbf{A precise asymmetry constant.} Under i.i.d.\ Gaussian weights and delocalized inputs, the ratio of transverse output energy from sign-flip versus magnitude perturbations converges to $\pi/(\pi-2) \approx 2.75$ in a two-layer ReLU + RMSNorm model (Theorem~3). While the Bussgang constant $\sqrt{2/\pi}$ for sign quantization is classical \citep{bussgang1952,rastegari2016}, to our knowledge this is the first closed-form quantification of the \emph{sign-vs-magnitude asymmetry ratio} through a ReLU + RMSNorm pipeline.

\item \textbf{A ternary quantization error characterization.} The angular alignment $\cos^2(y, y_T)$ between the output of a linear layer and its ternary-quantized counterpart converges to $2/\pi \approx 0.637$ as dimension grows (Theorem~4), corresponding to an angular error of $\sin^2\theta \to 1 - 2/\pi \approx 0.363$. The quantization error is sign-preserving (magnitude-type), with pre-ReLU radial fraction intermediate between sign-flips (${\approx}\,p$) and the idealized constant-$\delta$ model ($2/\pi$). The precise extension from the constant-$\delta$ model to true ternary error is addressed by the post-ReLU analysis (Contribution~4 below); the generalization to multi-layer settings and non-ReLU activations remains open (O1b, O3, O4 in Section~11).

\item \textbf{Experimental correction.} We test whether the single-layer factor compounds across depth and find that it does not: $C_L$ declines rather than grows with $L$. We identify outlier features (violation of the delocalization assumption) as the true bridge between the theoretical $2.75\times$ baseline and the substantially larger sign sensitivity observed in real models.

\item \textbf{Outlier leverage theory and post-ReLU ternary analysis.} We derive two complementary leverage results for outlier features (Appendix~\ref{app:outlier},~\ref{app:two-pop}): (i)~a per-entry leverage ratio $R(\alpha, n) = n\alpha^2(1 - 4\alpha/(3\pi))$ via an exact bivariate-normal calculation, and (ii)~a column-group leverage ratio $R_{\mathrm{col}} = [(1-\eta)/\eta]\cdot f(1-2\gamma)/f(2\gamma-1)$ for collective perturbations. At linear response ($p \leq 0.5\%$; Exp~E), count-matched NLL leverage stabilizes at ${\sim}10\times$, matching $n \cdot \mathbb{E}_{\mathrm{emp}}[\alpha^2] \approx 7$--$10$; Exps~F--G further confirm the $\alpha^2$ scaling (Spearman $\rho = 0.955$ and $0.927$). We also resolve the post-ReLU characterization of ternary error (Appendix~\ref{app:ternary-post}): $\mathcal{R}^{\mathrm{ternary}}_{\mathrm{post}} \approx 0.365$, within $0.4\%$ of the pre-ReLU value $1-2/\pi$, establishing that ReLU is approximately transparent to the radial/transverse decomposition for ternary error.
\end{enumerate}

\subsection{Related Work}

\paragraph{Bussgang's theorem and sign quantization.}
\citet{bussgang1952} established that $\sign(\cdot)$ applied to a zero-mean Gaussian preserves a fraction $\sqrt{2/\pi}$ of the linear correlation. \citet{rastegari2016} exploited this in XNOR-Net. Our Fact~1 re-derives the directional energy version: $\cos^2(w, \sign(w)) \to 2/\pi$.

\paragraph{Lottery Ticket Hypothesis and sign.}
\citet{frankle2019lottery} introduced the Lottery Ticket Hypothesis, establishing that sparse trainable subnetworks exist within randomly initialized networks. \citet{zhou2019signs} first demonstrated that sign patterns---rather than magnitudes---drive subnetwork trainability. \citet{oh2025sign} showed that sign alone can ``win the lottery'' in full-scale networks.

\paragraph{Ternary and 1-bit LLMs.}
Early work on weight binarization \citep{courbariaux2015binaryconnect,hubara2018quantized} established that neural networks can tolerate extreme quantization. \citet{wang2023bitnet} introduced BitNet with per-tensor sign quantization $\beta \cdot \sign(W)$ where $\beta = \frac{1}{mn}\|W\|_1$ (a single scalar for the entire weight matrix). \citet{ma2024bitnet} extended this to BitNet b1.58, which uses matrix-wide scaling with $\mathrm{RoundClip}$ to produce genuine ternary weights $\{-1, 0, +1\}$. Our Theorem~4 analyzes a per-row generalization (Section~\ref{sec:sign-quant}) that is asymptotically equivalent to the per-tensor scheme; the sign-preserving structure it characterizes is the shared mechanism underlying both architectures' quantization tolerance.

\paragraph{Post-training quantization.}
Modern post-training methods such as GPTQ \citep{frantar2023gptq}, AWQ \citep{lin2024awq}, and SmoothQuant \citep{xiao2023smoothquant} achieve sub-4-bit quantization for LLMs. These methods optimize for minimal output distortion; our analysis explains geometrically \emph{why} sign-preserving errors (as in ternary quantization) are inherently less damaging than sign-altering ones in pre-norm architectures.

\paragraph{Outlier features.}
\citet{dettmers2022} identified activation outliers in transformer hidden states. We show these outliers amplify sign sensitivity by violating the delocalization assumption underlying Theorem~3.

\paragraph{Direction-magnitude disentanglement.}
\citet{vardhan2026} documented a crossover double dissociation in LayerNorm-based Pythia models via $L_2$-matched perturbation analysis: angular perturbations dominate language-modeling damage while magnitude perturbations dominate syntactic processing damage. They note that this pattern may not generalize to other architectural families. Our work provides a theoretical mechanism---the $\pi/(\pi-2)$ asymmetry under ReLU + RMSNorm---that helps explain why RMSNorm architectures exhibit distinct direction-magnitude geometry compared to LayerNorm models.

\subsection{Paper Organization}

Section~2 establishes notation and assumptions. Section~3 proves that sign captures the dominant directional energy (Fact~1, Proposition~2). Section~4 analyzes RMSNorm geometry (Theorems~1--2). Section~5 shows single-layer insufficiency (Proposition~1). Section~6 presents the core $\pi/(\pi-2)$ result (Theorem~3). Section~7 addresses multi-layer behavior and outlier amplification. Section~8 characterizes ternary quantization error (Theorem~4). Section~9 presents unified experiments. Sections~10--11 discuss limitations and conclude. Appendix~A contains the complete proof of Theorem~3; Appendix~B provides the technical details of Theorem~4's vector-level concentration; Appendix~C gives the leading-order derivation in Proposition~2; Appendix~D contains additional experimental data; Appendix~E derives the single-entry outlier leverage ratio; Appendix~F develops the two-population column-group model; Appendix~G presents the post-ReLU ternary radial fraction derivation.

\section{Preliminaries and Notation}

\subsection{Notation}

\begin{table}[H]
\centering
\begin{tabular}{@{}lp{0.65\textwidth}@{}}
\toprule
Symbol & Definition \\
\midrule
$w \in \R^n$ & Weight vector (or row of weight matrix) \\
$W \in \R^{m \times n}$ & Weight matrix \\
$\xhat = x / \|x\|$ & Unit vector (direction of $x$) \\
$\RMSNorm(y) = \sqrt{m} \cdot y / \|y\|$ & Root-mean-square normalization \\
$P_{y^\perp} = I - \yhat\yhat^\top$ & Orthogonal projector onto $y^\perp$ \\
$\mathcal{R}(\Delta y) = (\Delta y \cdot \yhat)^2 / \|\Delta y\|^2$ & Radial fraction of perturbation \\
$\CV_p(d) = \mathrm{std}(d) / \mathrm{mean}(d)$ & Coefficient of variation \\
$\Corr(X,Y) = \E[XY]/\!\sqrt{\E[X^2]\E[Y^2]}$ & Uncentered correlation (cosine in $L^2$) \\
$\Delta W^{\mathrm{sign}}$ & Sign-flip perturbation \\
$\Delta W^{\mathrm{mag}}$ & Magnitude perturbation (sign-preserving) \\
$c(p)$ & Transverse energy ratio: asymptotic limit of $\|P_{y^\perp}\Delta y^{\mathrm{sign}}\|^2 / \|P_{y^\perp}\Delta y^{\mathrm{mag}}\|^2$ \\
\bottomrule
\end{tabular}
\end{table}

\subsection{RMSNorm}

RMSNorm \citep{zhang2019rmsnorm} simplifies LayerNorm \citep{ba2016layer} by omitting the re-centering operation. For $y \in \R^m \setminus \{0\}$:
\begin{equation}
\RMSNorm(y) = \frac{\sqrt{m}}{\|y\|}\, y = \sqrt{m} \cdot \yhat
\end{equation}
This maps all nonzero vectors to the sphere of radius $\sqrt{m}$, preserving direction and discarding norm.

\paragraph{Architectural role.} In a pre-norm Transformer, the block structure is $x_{l+1} = x_l + \text{FFN}(\RMSNorm(x_l))$, where RMSNorm normalizes the \emph{input} to each sublayer. Our toy model $\yhat = \RMSNorm(W_2\,\ReLU(W_1\xhat))$ analyzes the composition: the input $\xhat$ has already been normalized (serving as the pre-norm of the current layer), and the final $\RMSNorm$ represents the pre-norm of the \emph{next} layer applied to the current layer's FFN output. This captures the signal path between consecutive normalizations in a pre-norm architecture.

\subsection{Sign Quantization}\label{sec:sign-quant}

Given a weight matrix $W \in \R^{m \times n}$, we analyze the row-wise sign quantization:
\begin{equation}\label{eq:sign-quant}
W_T = \diag(s_1, \ldots, s_m) \cdot \sign(W), \quad s_i = \|w_i\|_1 / n
\end{equation}
where $\sign$ is applied entry-wise and $s_i$ is the mean absolute value of row~$i$. Each entry of $W_T$ takes values in $\{-s_i, +s_i\}$---a \textbf{binary} (two-level) quantization per row, not a genuine ternary $\{-1, 0, +1\}$ scheme.

\paragraph{Relationship to published schemes.} BitNet \citep{wang2023bitnet} uses a \emph{per-tensor} (single scalar) variant: $W_T = \beta \cdot \sign(W)$ with $\beta = \frac{1}{mn}\|W\|_1$, equivalent to setting $s_i = \beta$ for all rows. Our per-row formulation~\eqref{eq:sign-quant} is a slight generalization; since $s_i \to \E[|w_{ij}|]$ a.s.\ for each row as $n \to \infty$ (by SLLN), the per-row and per-tensor schemes become asymptotically equivalent under assumption~(A1'), and Theorem~\ref{thm:ternary}'s limit $\cos^2 \to 2/\pi$ holds for both.\footnote{The convergence rate may differ: per-row has $\Var(s_i) = O(1/n)$ with $m$ independent rows, while per-tensor $\beta$ averages over all $mn$ entries giving $\Var(\beta) = O(1/(mn))$. Both converge to the same limit.} \citet{ma2024bitnet} introduced a distinct scheme (BitNet b1.58): $W_T = \mathrm{RoundClip}(W/(\gamma + \epsilon),\, -1,\, 1)$ with a single matrix-wide scale $\gamma = \frac{1}{mn}\|W\|_1$, producing genuine ternary values $\{-1, 0, +1\}$ (zeros arise from rounding small weights). Our Theorem~\ref{thm:ternary} is proved for the sign scheme~\eqref{eq:sign-quant}; the sign-preserving structure it characterizes is common to both BitNet variants (for non-zero entries), but the contribution of BitNet b1.58's zero entries is not formally captured by our model.

\subsection{Assumptions}

We employ two sets of assumptions corresponding to different theorems.

\paragraph{Assumptions for Theorem~3} (two-layer ReLU + RMSNorm):
\begin{itemize}[leftmargin=2em]
\item[\textbf{(A1)}] $W_1 \in \R^{m \times n}$ and $W_2 \in \R^{d_{\mathrm{out}} \times m}$ have i.i.d.\ $N(0, 1/n)$ and $N(0, 1/m)$ entries respectively.
\item[\textbf{(A2)}] The input $\xhat \in S^{n-1}$ satisfies delocalization: $\|\xhat\|_\infty \to 0$ as $n \to \infty$.
\item[\textbf{(A3)}] Sign-flip perturbation: each entry of $W_1$ independently flips sign with probability $p \in (0, 1/2)$. The resulting perturbation satisfies $\E[\|\Delta W_1^{\mathrm{sign}}\|_F^2] = 4pm$.
\item[\textbf{(A4)}] Magnitude perturbation (sign-preserving): $W_1' = W_1 + \sign(W_1) \odot \delta$, where $\delta > 0$ is a scalar chosen so that $\|\Delta W_1^{\mathrm{mag}}\|_F^2 = 4pm$ (matching the sign-flip Frobenius norm). Explicitly, $\delta^2 = 4p/n$.
\item[\textbf{(A5)}] Dimensions satisfy $m/n \to \gamma \in (0, \infty)$, with $n, m, d_{\mathrm{out}} \to \infty$.
\item[\textbf{(A6)}] The activation function is ReLU: $\sigma(t) = \max(t, 0)$.
\end{itemize}

\paragraph{Assumptions for Theorem~4} (ternary quantization):
\begin{itemize}[leftmargin=2em]
\item[\textbf{(A1')}] $W \in \R^{m \times n}$ has i.i.d.\ entries $w_{ij} \sim N(0, 1/n)$. (The Gaussian assumption is required for the $2/\pi$ limit; for general i.i.d.\ symmetric distributions with $0 < \E[w^2] < \infty$, $\E[w^4] < \infty$, and $\Pr(w=0) = 0$, the more general result $\cos^2 \to (\E|w|)^2/\E[w^2]$ holds.)
\item[\textbf{(A2')}] $\xhat \sim \Unif(S^{n-1})$, independent of $W$. (Equivalently, any $x$ with $\|\xhat\|_\infty = o(1)$ w.h.p.)
\item[\textbf{(A3')}] $m, n \to \infty$ with $m/n = O(1)$.
\item[\textbf{(A4')}] $w_{ij} \neq 0$ a.s.\ (automatically satisfied for Gaussian).
\end{itemize}

\subsection{Perturbation Types}

\begin{definition}[Perturbation type]\label{def:perturbation-type}
A perturbation $\Delta W$ applied to a pre-norm layer is called \textbf{sign-flip-type} if it reverses the sign of a subset of entries ($\sign(W'_{ij}) \neq \sign(W_{ij})$ for perturbed entries), and \textbf{magnitude-type} (sign-preserving) if $\sign(W + \Delta W) = \sign(W)$ entry-wise.

\end{definition}

\noindent\textbf{Functional characterization} (consequence of Theorem~3, under the constant-$\delta$ model (A4) specifically).
Under assumptions (A1)--(A6), as $n \to \infty$ with $p$ fixed:
\begin{itemize}
  \item \emph{magnitude-type (constant-$\delta$):}\quad $\mathcal{R}(\Delta y) \xrightarrow{P} 2/\pi + O(\sqrt{p})$
  \item \emph{sign-flip-type:}\quad $\mathcal{R}(\Delta y) \xrightarrow{P} p - 4p^{3/2}/(3\pi) + O(p^{5/2})$
\end{itemize}
The asymptotic separation of these radial fractions ($2/\pi$ vs.\ $O(p)$) is the geometric source of the asymmetry quantified in Theorem~3. For general magnitude-type perturbations (e.g., true ternary quantization error), the radial fraction differs; see Section~8.3.

\section{Sign Captures Directional Energy}

This section establishes that sign patterns dominate the directional content of weight vectors, while magnitude with random signs carries vanishing directional information.

\subsection{The Bussgang Constant}

\begin{fact}[Sign directional coverage; \citealt{bussgang1952}]\label{fact:bussgang}
Let $w \in \R^n$ with i.i.d.\ symmetric components satisfying $\E[|w_1|], \E[w_1^2] < \infty$, $\E[w_1^2] > 0$, and $\Pr(w_j = 0) = 0$. Then:
\begin{equation}
\cos^2(w, \sign(w)) \xrightarrow{\mathrm{a.s.}} \frac{(\E[|w_1|])^2}{\E[w_1^2]} \quad (n \to \infty)
\end{equation}
For Gaussian components $w_j \sim N(0, \sigma^2)$, the limit is the \textbf{Bussgang constant}:
\begin{equation}
\cos^2(w, \sign(w)) \xrightarrow{\mathrm{a.s.}} \frac{2}{\pi} \approx 0.6366
\end{equation}
\end{fact}

\begin{proof}
Write $\cos(w, \sign(w)) = w^\top \sign(w) / (\|w\|_2 \cdot \|\sign(w)\|_2)$. The numerator equals $\|w\|_1$ and $\|\sign(w)\|_2 = \sqrt{n}$ a.s., so $\cos(w, \sign(w)) = \|w\|_1 / (\sqrt{n} \cdot \|w\|_2)$. By the SLLN, $\|w\|_1/n \to \E[|w_1|]$ and $\|w\|_2^2/n \to \E[w_1^2]$ a.s. The continuous mapping theorem (applied to $g(a,b) = a/\sqrt{b}$, continuous at $b > 0$) yields:
\[
\cos(w, \sign(w)) \xrightarrow{\mathrm{a.s.}} \frac{\E[|w_1|]}{\sqrt{\E[w_1^2]}}
\]
Squaring gives the general result. For the Gaussian case, $\E[|w_1|] = \sigma\sqrt{2/\pi}$ and $\E[w_1^2] = \sigma^2$, yielding $\cos^2 \to 2/\pi$.
\end{proof}

\paragraph{Convergence rate.} Under the additional assumption $\E[w_1^4] < \infty$, the Delta method applied to the bivariate CLT for $(\|w\|_1/n,\, \|w\|_2^2/n)$ gives $\cos^2 - 2/\pi = O_P(1/\sqrt{n})$.

\subsection{Magnitude with Random Signs is Uninformative}

\setcounter{proposition}{1}
\begin{proposition}\label{prop:magnitude-uninformative}
Let $w \in \R^n$ with i.i.d.\ symmetric components, $0 < \E[w_1^2] < \infty$, $\E[w_1^8] < \infty$. Let $s \in \{\pm 1\}^n$ be an independent Rademacher vector. Then:
\begin{equation}
\E\!\left[\cos^2(w,\; s \odot |w|)\right] = \frac{\E[w_1^4]}{n \cdot (\E[w_1^2])^2} \cdot (1 + o(1)) \xrightarrow{n \to \infty} 0
\end{equation}
For Gaussian components: $\E[\cos^2] = 3/(n+2)$ (exact for all $n \geq 1$).
\end{proposition}

\begin{proof}
\textbf{Step 1.} Using $w_j |w_j| = \sign(w_j) w_j^2$ and $\|s \odot |w|\| = \|w\|$:
\[
\cos(w, s \odot |w|) = \frac{\sum_j c_j w_j^2}{\sum_j w_j^2}
\]
where $c_j := s_j \cdot \sign(w_j)$ are i.i.d.\ Rademacher, independent of $(|w_1|, \ldots, |w_n|)$.

\textbf{Step 2.} Conditioning on $|w|$ and using $\E[c_j c_k] = \delta_{jk}$:
\[
\E[\cos^2 \mid |w|] = \frac{\sum_j w_j^4}{(\sum_j w_j^2)^2}
\]

\textbf{Step 3.} By SLLN, $\sum_j w_j^4 / n \to \mu_4$ and $\sum_j w_j^2 / n \to \mu_2$ a.s. A good-event/bad-event argument (defining $G_n = \{S_n \geq n\mu_2/2\}$, using uniform integrability on $G_n$ and Chernoff bound $\Pr(G_n^c) \leq e^{-cn}$ on $G_n^c$) upgrades a.s.\ convergence to $L^1$, giving $\E[\cos^2] = \mu_4/(n\mu_2^2) \cdot (1 + o(1))$. (The uniform integrability step requires $\E[w_1^8] < \infty$; full derivation in Appendix~C.)

\textbf{Step 4 (Gaussian exact).} For $w_j \sim N(0, \sigma^2)$, the normalized squares $(w_1^2, \ldots, w_n^2)/\sum_k w_k^2 \sim \Dir(1/2, \ldots, 1/2)$ by the Gamma-Dirichlet theorem. Using $\E[u_j^2] = 3/(n(n+2))$ for $\Dir(1/2, \ldots, 1/2)$: $\E[\cos^2] = n \cdot 3/(n(n+2)) = 3/(n+2)$.
\end{proof}

\subsection{Corollary: Sign is the Dominant Directional Component}

\begin{corollary}[Sign dominates direction]\label{cor:sign-dominates}
For i.i.d.\ symmetric weight vectors $w \in \R^n$ with $\E[w_1^8] < \infty$ (inheriting the moment condition from Proposition~\ref{prop:magnitude-uninformative}):
\begin{itemize}
\item The sign pattern $\sign(w)$ captures $2/\pi \approx 63.7\%$ of directional energy (Fact~\ref{fact:bussgang});
\item Magnitude with random signs provides only $O(1/n) \to 0$ directional information (Proposition~\ref{prop:magnitude-uninformative}).
\end{itemize}
The two are complementary: sign determines the ``directional region,'' while magnitude provides fine-grained refinement within that region. This decomposition motivates why sign-flip perturbations are far more damaging to network outputs than magnitude perturbations of equal norm.
\end{corollary}

\section{RMSNorm Geometry}

We now analyze how RMSNorm interacts with perturbations. Theorem~\ref{thm:per-row-scaling} quantifies scale absorption; Theorem~\ref{thm:transverse-proj} reveals the underlying geometric mechanism---a transverse projection.

\subsection{Per-Row Scaling Bound}

\begin{theorem}[Per-Row Scaling Bound]\label{thm:per-row-scaling}
Let $y \in \R^m \setminus \{0\}$, $D = \diag(d_1, \ldots, d_m)$ with $d_i \geq 0$ and $Dy \neq 0$. Define the energy-weighted probability $p_i = y_i^2/\|y\|^2$ and corresponding mean and variance:
\[
\bar{d}_p = \sum_{i=1}^m p_i d_i, \qquad \Var_p(d) = \sum_{i=1}^m p_i(d_i - \bar{d}_p)^2
\]
Then:
\begin{equation}
\frac{\|\RMSNorm(Dy) - \RMSNorm(y)\|}{\|\RMSNorm(y)\|} \leq \frac{\sqrt{\Var_p(d)}}{\bar{d}_p} = \CV_p(d)
\end{equation}
When $D = \alpha I$ (uniform scaling), $\Var_p(d) = 0$ and the error vanishes exactly.
\end{theorem}

\begin{proof}
\textbf{Step 1.} Set $y' = Dy$. The difference vector has components $\Delta_i = \RMSNorm(y)_i \cdot (d_i/r - 1)$ where $r = \|Dy\|/\|y\|$.

\textbf{Step 2.} The squared relative error becomes $\|\Delta\|^2/\|\RMSNorm(y)\|^2 = \sum_i p_i(d_i/r - 1)^2$, which simplifies to $2(1 - \bar{d}_p/r)$ using $r^2 = \E_p[d^2] = \bar{d}_p^2 + \Var_p(d)$.

\textbf{Step 3.} Let $\theta := \angle(\RMSNorm(Dy), \RMSNorm(y))$. The geometric identity $\cos\theta = \bar{d}_p/r$ gives the exact expression as the spherical chord-length formula $2(1-\cos\theta)$.

\textbf{Step 4.} Applying $1 - 1/\sqrt{1+x} \leq x/2$ for $x = \Var_p(d)/\bar{d}_p^2 \geq 0$:
\[
2\left(1 - \frac{\bar{d}_p}{r}\right) \leq \frac{\Var_p(d)}{\bar{d}_p^2} = \CV_p(d)^2
\]
Taking square roots completes the proof.
\end{proof}

\paragraph{Tightness.} The bound is extremely tight for small CV: at $\CV = 0.1$, the exact value is $0.9926\%$ vs.\ the bound $1.0\%$ (slack ratio $1.0075$).

\subsection{Fr\'{e}chet Derivative: The Transverse Projector}

\begin{theorem}[RMSNorm Transverse Projection]\label{thm:transverse-proj}
Let $f(y) = \RMSNorm(y) = \sqrt{m} \cdot y/\|y\|$. The Fr\'{e}chet derivative at $y \neq 0$ is:
\begin{equation}\label{eq:frechet}
Df(y) = \frac{\sqrt{m}}{\|y\|}\, P_{y^\perp}
\end{equation}
where $P_{y^\perp} = I_m - \yhat\yhat^\top$ is the orthogonal projector onto the subspace perpendicular to~$y$.
\end{theorem}

\begin{proof}
Computing $\partial f_i / \partial y_j = (\sqrt{m}/\|y\|)(\delta_{ij} - y_i y_j/\|y\|^2)$, the Jacobian matrix is $(\sqrt{m}/\|y\|)(I - \yhat\yhat^\top) = (\sqrt{m}/\|y\|)P_{y^\perp}$. Since $f$ is $C^\infty$ on $\R^m \setminus \{0\}$, the Jacobian equals the Fr\'{e}chet derivative.
\end{proof}

\begin{corollary}[Perturbation Decomposition]\label{cor:perturbation-decomp}
For a perturbation $\Delta y$:
\begin{equation}
f(y + \Delta y) - f(y) = \frac{\sqrt{m}}{\|y\|}\, P_{y^\perp}\Delta y + O\!\left(\frac{\sqrt{m}\|\Delta y\|^2}{\|y\|^2}\right)
\end{equation}
The radial component $(\yhat^\top \Delta y)\yhat$ is eliminated at first order; the transverse component $P_{y^\perp}\Delta y$ passes through completely. This is the mechanism by which RMSNorm acts as a \textbf{radial filter}: it absorbs perturbations aligned with the output direction and preserves those orthogonal to it.
\end{corollary}

\begin{corollary}[Second-Order Remainder]\label{cor:second-order}
For $\|\Delta y\| \leq \eta\|y\|$ with $0 < \eta < 1$:
\begin{equation}
\|f(y + \Delta y) - f(y) - Df(y)\Delta y\| \leq \frac{3\sqrt{m}}{(1-\eta)^2 \|y\|^2}\|\Delta y\|^2
\end{equation}
The relative remainder $\|R_2\|/\|f(y)\| \leq 3\eta^2/(1-\eta)^2$ is dimension-independent. For $\eta = 0.05$: $\leq 0.8\%$; for $\eta = 0.10$: $\leq 3.7\%$.
\end{corollary}

\begin{proof}
By the fundamental theorem of calculus, $R_2 = \int_0^1 [Df(y+t\Delta y) - Df(y)]\Delta y\, dt$. Differentiating the Jacobian $Df(y) = (\sqrt{m}/\|y\|)P_{y^\perp}$ gives the Hessian bilinear form:
\[
D^2f(y)[h,h] = -\frac{\sqrt{m}}{\|y\|^2}\bigl[2(\yhat \cdot h)\,P_{y^\perp} h + \|P_{y^\perp} h\|^2\,\yhat\bigr]
\]
The two bracketed terms are orthogonal ($P_{y^\perp}h \perp \yhat$), so by the triangle inequality $\sqrt{a^2+b^2} \leq |a|+|b|$:
\[
\|D^2f(y)[h,h]\| \leq \frac{\sqrt{m}}{\|y\|^2}\bigl(2|\yhat \cdot h|\,\|P_{y^\perp}h\| + \|P_{y^\perp}h\|^2\bigr) \leq \frac{3\sqrt{m}\|h\|^2}{\|y\|^2}
\]
where the last step uses $|\yhat \cdot h| \leq \|h\|$ and $\|P_{y^\perp}h\| \leq \|h\|$. For any $z = y + t\Delta y$ on the segment, $\|z\| \geq (1-\eta)\|y\|$, so $\|D^2f(z)\|_{\mathrm{op}} \leq 3\sqrt{m}/((1-\eta)^2\|y\|^2)$. By the mean-value inequality on the derivative:
\[
\|Df(y+t\Delta y) - Df(y)\|_{\mathrm{op}} \leq t\|\Delta y\| \cdot \sup_{z}\|D^2f(z)\|_{\mathrm{op}} \leq \frac{3\sqrt{m}\,t\|\Delta y\|}{(1-\eta)^2\|y\|^2}
\]
Substituting into the integral representation and bounding $\sup_{t \in [0,1]} t \leq 1$ yields the stated bound. (The constant~$3$ is not tight; using the integral form $\int_0^1 (1-t)\, dt = 1/2$ and the AM-GM refinement $|\yhat \cdot h|\,\|P_{y^\perp}h\| \leq \|h\|^2/2$ yields a sharper constant of~$1$. We retain~$3$ for a self-contained single-step argument.)
\end{proof}

\subsection{Differential Implications for Sign vs.\ Magnitude}

The projector $P_{y^\perp}$ acts differentially on the two perturbation types:

\begin{table}[!ht]
\centering
\caption{Differential implications of the transverse projector $P_{y^\perp}$ on three perturbation types.}
\label{tab:differential}
\begin{tabular}{@{}llll@{}}
\toprule
Perturbation type & Structure of $\Delta y$ & Radial fraction $\mathcal{R}$ & After $P_{y^\perp}$ \\
\midrule
Uniform scaling $\epsilon W$ & $\Delta y \propto y$ (pure radial) & $\approx 1$ & $\approx 0$ (absorbed) \\
Sign-flip & Large transverse component & $\approx p$ (small) & $(1-p)$ preserved \\
Magnitude noise & Mostly radial (Bussgang) & $\approx 2/\pi$ (large) & $(1-2/\pi)$ preserved \\
\bottomrule
\end{tabular}
\end{table}

The radial fraction $\mathcal{R}^{\mathrm{mag}} \approx 2/\pi$ vs.\ $\mathcal{R}^{\mathrm{sign}} \approx p$ is the geometric origin of the $\pi/(\pi-2)$ asymmetry factor derived in Section~6.

\section{Single-Layer Equivalence and Its Insufficiency}

\subsection{Energy Equivalence}

\setcounter{proposition}{0}
\begin{proposition}[Single-Layer Energy Equivalence]\label{prop:single-layer}
Let $\Delta W^{(1)}, \Delta W^{(2)} \in \R^{m \times n}$ be given matrices (deterministic, or any fixed realization of random matrices) with $\|\Delta W^{(1)}\|_F = \|\Delta W^{(2)}\|_F$. For $\xhat \sim \Unif(S^{n-1})$:
\begin{equation}
\E_{\xhat}\!\left[\|\Delta W^{(1)} \xhat\|^2\right] = \E_{\xhat}\!\left[\|\Delta W^{(2)} \xhat\|^2\right] = \frac{\|\Delta W^{(1)}\|_F^2}{n}
\end{equation}
(The expectation is over $\xhat$ only; $\Delta W^{(k)}$ are held fixed.)
\end{proposition}

\begin{proof}
For any $M \in \R^{m \times n}$ and $\xhat \sim \Unif(S^{n-1})$: $\E[\|M\xhat\|^2] = \E[\xhat^\top M^\top M \xhat] = \tr(M^\top M)/n = \|M\|_F^2/n$, using the spherical identity $\E[\xhat\xhat^\top] = I_n/n$. Equal Frobenius norms imply equal output energies.
\end{proof}

\subsection{Why Single-Layer Cannot Explain the Asymmetry}

Despite equal total energy, the two perturbation types already exhibit different \textbf{directional structure} relative to the unperturbed output $y = W\xhat$:

\begin{itemize}
\item \textbf{Sign perturbation}: $\E[\Delta y^{\mathrm{sign}} \mid W, \xhat] = -2p \cdot y$ (conditional mean is pure radial), giving radial fraction $\mathcal{R}^{\mathrm{sign}} \approx p$.
\item \textbf{Magnitude perturbation} (additive independent model): $\Cov(\Delta y^{\mathrm{mag}}) = (4p/n)\, I_m$ (isotropic, with variance $\sigma_\delta^2 = \delta^2 = 4p/n$ matching the sign-flip Frobenius norm), giving radial fraction $\mathcal{R}^{\mathrm{mag}} = 1/m$.
\end{itemize}

Under this additive model, a single-layer RMSNorm predicts (as $m \to \infty$, so $\mathcal{R}^{\mathrm{mag}} = 1/m \to 0$):
\begin{equation}
\frac{\E[\|P_{y^\perp}\Delta y^{\mathrm{mag}}\|^2]}{\E[\|P_{y^\perp}\Delta y^{\mathrm{sign}}\|^2]} \to \frac{1}{1-p} > 1
\end{equation}

This predicts magnitude perturbations are \emph{more} damaging than sign perturbations after RMSNorm---opposite to experimental observation. The resolution requires: (i)~adopting the \textbf{sign-preserving} magnitude model (assumption~(A4), where $\Delta W^{\mathrm{mag}} = \sign(W) \odot \delta$ depends on $W$ through $\sign(W)$, creating Bussgang correlation $\rho = \sqrt{2/\pi}$ between $\Delta z$ and $z$) rather than the additive independent model considered above; and (ii)~introducing ReLU to create hidden-space geometric asymmetry. Both ingredients are provided by Theorem~3 (Section~6). The additive independent model of this subsection is solely a motivating counterexample; all subsequent analysis uses the sign-preserving model~(A4).

\section{Core Result: The \texorpdfstring{$\pi/(\pi-2)$}{pi/(pi-2)} Asymmetry}

\subsection{Model Setup}

Consider a simplified two-layer pre-norm block:
\begin{equation}
\xhat = x/\|x\| \in S^{n-1}, \quad z = W_1\xhat, \quad a = \ReLU(z), \quad y = W_2 a, \quad \yhat = \RMSNorm(y)
\end{equation}
with $W_1 \in \R^{m \times n}$ (i.i.d.\ $N(0,1/n)$), $W_2 \in \R^{d_{\mathrm{out}} \times m}$ (i.i.d.\ $N(0,1/m)$), independent of each other.

\paragraph{Perturbation models} (matched expected Frobenius norm $\E[\|\Delta W_1\|_F^2] = 4pm$; for sign-flips this is an expectation with variance $O(pm)$ that concentrates for large $mn$):
\begin{itemize}
\item \emph{Sign-flip}: Each entry of $W_1$ independently flips sign with probability~$p$.
\item \emph{Magnitude (sign-preserving)}: $W_1' = W_1 + \sign(W_1) \odot \delta$, with $\delta^2 = 4p/n$.
\end{itemize}

\subsection{Theorem Statement}

\begin{theorem}[ReLU-Induced Differential Suppression]\label{thm:core}
Under assumptions \textup{(A1)--(A6)}:

\textup{\textbf{(Asymptotic limit)}} For each fixed $p \in (0, 1/2)$, as $n, m, d_{\mathrm{out}} \to \infty$ with $m/n \to \gamma \in (0, \infty)$:
\begin{equation}
c(p) := \frac{\|P_{y^\perp}\Delta y^{\mathrm{sign}}\|^2}{\|P_{y^\perp}\Delta y^{\mathrm{mag}}\|^2} \xrightarrow{P} \frac{1 - \mathcal{R}^{\mathrm{sign}}(p)}{1 - \mathcal{R}^{\mathrm{mag}}(p)} \cdot R_A(p)
\end{equation}
where $R_A(p) = \E[\|\Delta y^{\mathrm{sign}}\|^2]/\E[\|\Delta y^{\mathrm{mag}}\|^2]$ is the pre-RMSNorm energy ratio, and $\mathcal{R}$ denotes the radial fraction.

\textup{\textbf{(Small-$p$ behavior)}} Taking $p \to 0$ in the resulting limit:
\begin{equation}\label{eq:core-result}
\boxed{c(p) = \frac{\pi}{\pi - 2} + O(\sqrt{p}) \quad \text{as } p \to 0, \qquad \frac{\pi}{\pi-2} \approx 2.75}
\end{equation}
The leading constant $\pi/(\pi-2)$ is the theorem's primary rigorous claim. The $O(\sqrt{p})$ remainder is bounded but its coefficient is not explicitly determined; a structural decomposition is given in Remark~\ref{rem:structural}.
\end{theorem}

\begin{remark}[Structural decomposition]\label{rem:structural}
The $O(\sqrt{p})$ remainder in Theorem~\ref{thm:core} admits a \textbf{factored decomposition} with partially identified coefficients:
\[
c(p) = \frac{\pi}{\pi - 2}\,(1-p)\!\left(1 - \frac{4\sqrt{p}}{3\pi}\right)\!\bigl(1 + \varepsilon(p)\bigr)
\]
where each factor has a clear provenance: $(1-p)$ from the radial fraction ratio (Part~B), $(1-4\sqrt{p}/(3\pi))$ from the energy ratio $R_A(p)$ (Part~A; the coefficient $4/(3\pi)$ is exact from Steps~A4--A6, with an $O(\sqrt{p})$ multiplicative remainder from Step~A8), and $(1+\varepsilon(p))$ from Part~B's boundary corrections. Both the factor $(1-4\sqrt{p}/(3\pi))$ and $\varepsilon(p)$ contribute at $O(\sqrt{p})$, so the factored form is \emph{not} a traditional asymptotic expansion with separated orders; it is a structural decomposition in which the coefficient $4/(3\pi)$ is exact while $\varepsilon(p)$'s coefficient remains bounded but undetermined ($|\varepsilon(p)| \leq C\sqrt{p}$ for some $C > 0$; empirically $C \lesssim 0.2$ from Table~6). Note that the factored form retains an explicit $O(p)$ term via $(1-p)$ that the additive bound $O(\sqrt{p})$ does not isolate; the two forms are compatible as upper-bound statements but serve different goals (numerical comparison vs.\ asymptotic order). This factored form is used for the ``theory'' column in Tables~2 and~6.
\end{remark}

\begin{remark}[Limit ordering and finite-$n$ interpretation]\label{rem:double-limit}
Theorem~\ref{thm:core} involves a \textbf{double limit}: first $n \to \infty$ (for each fixed~$p$), then $p \to 0$ for the leading-order expansion. The bivariate Gaussian approximation (Step~A3 in Appendix~A) has Berry--Esseen error $O(\|\xhat\|_\infty) = O(1/\sqrt{n})$ under delocalization. Consequently:
\begin{itemize}
\item \emph{Convergence mode}: The left-hand side of (13) is a sample ratio (a single draw of $\xhat, W_1, W_2$). The proof computes $\E[\|\cdot\|^2]/\E[\|\cdot\|^2]$ (ratio of expectations over $\xhat$). Since both $\|\Delta y^{\mathrm{sign}}\|^2$ and $\|\Delta y^{\mathrm{mag}}\|^2$ are averages of $d_{\mathrm{out}}$ i.i.d.\ projections (conditioned on $(W_1, \xhat)$; the randomness comes from $W_2$'s rows), the SLLN gives concentration of each around its conditional mean as $d_{\mathrm{out}} \to \infty$. The denominator concentrates around a positive constant (since $\E[\|\Delta y^{\mathrm{mag}}\|^2] > 0$), so by Slutsky's theorem the ratio converges in probability to the ratio of expectations.
\item The primary claim $\lim_{p \to 0} c(p) = \pi/(\pi-2)$ is rigorous: it is the $n \to \infty$ limit for any fixed small~$p$, followed by $p \to 0$ in the resulting deterministic function. The leading constant $\pi/(\pi-2)$ relies on exact population moment identities ($\Corr(z_i, \Delta z_i) = \sqrt{2/\pi}$, Lemma~A.1) that hold for all~$n$, so it is not affected by CLT error.
\item The $\sqrt{p}$ correction in the factored expansion (Remark~\ref{rem:structural}) requires the CLT's $O(1/\sqrt{n})$ error to be negligible relative to $\sqrt{p}$, i.e., $p \gg 1/n$---easily satisfied at all experimental values ($p \geq 0.01$, $n = 2048$: $1/n \approx 5 \times 10^{-4}$).
\item In experimental tables (fixed $n{=}2048$, varying~$p$), the observed deviations from theory ($1$--$6\%$, Table~6) may reflect both $\varepsilon(p)$ and residual finite-$n$ effects; the two contributions are not disentangled.
\end{itemize}
\end{remark}

\subsection{Proof Sketch}

The proof decomposes into two parts: (A)~computing the total energy ratio $R_A(p)$ before RMSNorm, and (B)~computing the radial fractions that determine how much energy survives RMSNorm's transverse projection.

\paragraph{Part A: Energy ratio before RMSNorm.}

The key computation is the expected squared activation change $\E[(\Delta a_i)^2]$ for each neuron~$i$. For sign-flip perturbations, there are two contributions:

\begin{enumerate}
\item \emph{Gate-flip energy} $\beta_{\mathrm{flip}}$: When the sign of $z_i$ changes (neuron switches on/off). Using the bivariate Gaussian CDF with correlation $\rho = 1-2p$:
\begin{equation}
\beta_{\mathrm{flip}} = \frac{8p^{3/2}}{3\pi} + O(p^{5/2})
\end{equation}

\item \emph{Smooth energy} $\beta_{\mathrm{smooth}}$: When the neuron remains in the same activation state. Via the Price theorem / ReLU product formula $\E[\max(X,0)\max(X',0)] = (2\pi)^{-1}[\sqrt{1-\rho^2} + \rho(\pi/2 + \arcsin\rho)]$:
\begin{equation}
\beta_{\mathrm{smooth}} = 2p - \frac{8p^{3/2}}{\pi} + O(p^{5/2})
\end{equation}
\end{enumerate}

For magnitude perturbations, $\E[(\Delta a_i^{\mathrm{mag}})^2] = \E[\mathbf{1}_{z_i > 0}(\Delta z_i)^2] = (1/2) \cdot \delta^2 = 2p/n$ (half-space symmetry: $\Pr(z_i > 0) = 1/2$ for symmetric $z_i$; $\delta^2 = 4p/n$ by~(A4)). This gives:
\begin{equation}
R_A(p) = \left(1 - \frac{4\sqrt{p}}{3\pi}\right)\!\bigl(1 + O(\sqrt{p})\bigr)
\end{equation}
where the leading factor has an exact coefficient from Steps~A4--A6, and the multiplicative remainder $O(\sqrt{p})$ arises from boundary gate-flip corrections in Step~A7.

\paragraph{Part B: Radial fraction analysis.}

Using SLLN concentration and the random projection angle-preservation lemma ($W_2$ preserves angles as $d_{\mathrm{out}} \to \infty$):
\begin{itemize}
\item \emph{Sign perturbation}: The reflection symmetry identity $\E[\Delta a_i \cdot a_i] = -\E[(\Delta a_i)^2]/2$ gives $\mathcal{R}^{\mathrm{sign}} \to p - 4p^{3/2}/(3\pi) + O(p^{5/2})$.
\item \emph{Magnitude perturbation}: By Lemma~\ref{lem:bussgang-relu} (Bussgang-ReLU correlation; Appendix~A.2), the Bussgang correlation $\Corr(\Delta z_i, z_i) = \sqrt{2/\pi}$ propagates through ReLU to give $\mathcal{R}^{\mathrm{mag}} \to 2/\pi + O(\sqrt{p})$.
\end{itemize}

Combining via $\|P_{y^\perp}\Delta y\|^2 = (1-\mathcal{R})\|\Delta y\|^2$:
\[
c(p) = \frac{\|P_{y^\perp}\Delta y^{\mathrm{sign}}\|^2}{\|P_{y^\perp}\Delta y^{\mathrm{mag}}\|^2} = \frac{1-\mathcal{R}^{\mathrm{sign}}}{1-\mathcal{R}^{\mathrm{mag}}} \cdot \underbrace{\frac{\|\Delta y^{\mathrm{sign}}\|^2}{\|\Delta y^{\mathrm{mag}}\|^2}}_{R_A(p)} \to \frac{1-p}{1-2/\pi} \cdot R_A(p) = \frac{\pi(1-p)}{\pi-2}\cdot R_A(p).
\]
The complete proof is given in Appendix~A.

\subsection{Intuition: ``Sign Structure + ReLU Create, RMSNorm Exposes''}

The factor $\pi/(\pi-2)$ arises from two complementary mechanisms:
\begin{enumerate}
\item \textbf{Sign-preserving structure provides Bussgang alignment; ReLU propagates it}: The correlation $\rho = \sqrt{2/\pi}$ between $\sign(W)$ and $W$ is a pre-existing property of the weight geometry---it is \emph{introduced} by the sign-preserving perturbation structure $\Delta w_{ij} \propto \sign(w_{ij})$, not by ReLU itself. ReLU's role is to \emph{propagate} this alignment faithfully through its half-space partition (Lemma~\ref{lem:bussgang-relu}), yielding a post-activation radial fraction $\mathcal{R}^{\mathrm{mag}} \to 2/\pi$, while simultaneously giving sign-flip perturbations a near-zero radial fraction $\mathcal{R}^{\mathrm{sign}} \to p$. Without ReLU (linear activation), no asymmetry exists (Proposition~\ref{prop:single-layer}).

\item \textbf{RMSNorm exposes the asymmetry}: By projecting out the radial component (Theorem~\ref{thm:transverse-proj}), RMSNorm retains only $(1-2/\pi) \approx 36.3\%$ of magnitude perturbation energy but $(1-p) \approx 99\%$ of sign perturbation energy (for small~$p$). The ratio is $\pi/(\pi-2)$.
\end{enumerate}

\subsection{Numerical Verification}

Table~\ref{tab:theorem3-verification} and Figure~\ref{fig:theorem3} compare the theoretical prediction against Monte Carlo simulation (V1 architecture, $n{=}m{=}d_{\mathrm{out}}{=}2048$, 20 seeds, 10,000 samples per seed). Agreement is within 1--6\% across $p \in [0.01, 0.10]$.

\begin{table}[H]
\centering
\begin{tabular}{@{}cccc@{}}
\toprule
$p$ & $c(p)$ theory & $c(p)$ measured (V1, $n{=}2048$) & Deviation \\
\midrule
0.01 & 2.610 & $2.639 \pm 0.012$ & $+1.1\%$ \\
0.02 & 2.538 & $2.583 \pm 0.010$ & $+1.8\%$ \\
0.05 & 2.378 & $2.465 \pm 0.010$ & $+3.7\%$ \\
0.10 & 2.178 & $2.299 \pm 0.008$ & $+5.6\%$ \\
\bottomrule
\end{tabular}
\caption{The V1 architecture (correct two-layer model) matches theory within 1--6\% across all tested $p$ values. Theory values are computed from the structural decomposition (Remark~\ref{rem:structural}): $c(p) = \frac{\pi}{\pi-2}(1-p)(1-\frac{4\sqrt{p}}{3\pi})$ with exact $\beta_{\mathrm{flip}}, \beta_{\mathrm{smooth}}$ coefficients; the multiplicative remainder $\varepsilon(p)$ (whose coefficient is the same order $O(\sqrt{p})$ as the structural correction $4\sqrt{p}/(3\pi)$) is \emph{not} included. The systematic positive bias ($+1$--$6\%$, increasing with~$p$) is consistent with $\varepsilon(p) > 0$, though finite-$n$ CLT residuals ($O(1/\sqrt{n}) \approx 2\%$ at $n{=}2048$) are of comparable magnitude and are not disentangled (Remark~\ref{rem:double-limit}).}
\label{tab:theorem3-verification}
\end{table}

\begin{figure}[H]
\centering
\includegraphics[width=0.85\textwidth]{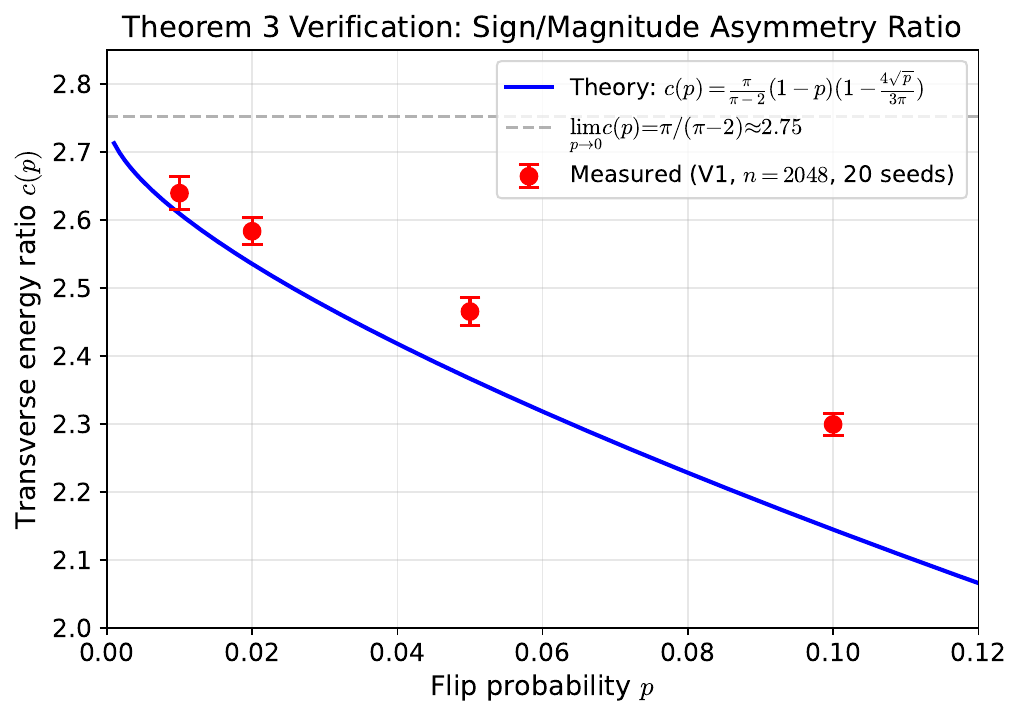}
\caption{Theorem~3 verification. Blue curve: leading-order theoretical prediction $c(p) = \frac{\pi}{\pi-2}(1-p)(1 - \frac{4\sqrt{p}}{3\pi})$ (excluding the multiplicative $\varepsilon(p)$ correction of Remark~\ref{rem:structural}). Red points: measured values from the V1 architecture ($n{=}2048$, 20 seeds, $2\sigma$ error bars). The dashed line marks the asymptotic limit $\pi/(\pi-2) \approx 2.75$. Theory and measurement agree within 1--6\%.}
\label{fig:theorem3}
\end{figure}

\section{Beyond the Toy Regime}

Theorem~\ref{thm:core} establishes a single-layer asymmetry of $\pi/(\pi-2) \approx 2.75$. Real models exhibit substantially larger sign sensitivity (with outlier-targeted sign-flip $\Delta$PPL ratios exceeding $10^3\times$; see Section~9.5). This section addresses the gap.

\subsection{Multi-Layer Compounding Does Not Hold}

\begin{conjecture}[Multi-layer compounding; refuted in Section~9.4]
Since Theorem~\ref{thm:core} establishes a per-layer sign-to-magnitude energy ratio $\pi/(\pi-2) > 1$, and perturbation amplification factors commonly compound multiplicatively across layers in deep networks (e.g., Lipschitz stability bounds), a natural prima facie hypothesis is that the single-layer factor compounds across $L$ layers:
\[
C_L \approx \left(\frac{\pi}{\pi-2}\right)^{\!\beta L}, \quad \beta \in (0,1]
\]
This predicts (P1)~$C_L$ increases monotonically with $L$, and (P2)~$\log C_L$ is linear in $L$ with positive slope.
\end{conjecture}

\paragraph{Experimental refutation.} Three independent experiment groups contradict both predictions (Table~\ref{tab:multilayer}): in all cases $C_L$ \emph{decreases} with depth---opposite to the predicted monotone increase for any $\beta > 0$. The strong version ($\beta=1$) predicted $C_6 \approx 437$; the observed values are $C_6 \in [1.3, 3.1]$, refuting it decisively. Weaker versions with $\beta \to 0$ predict near-constant $C_L$ (e.g., $\beta = 0.02$ gives $C_6 \approx 1.13$), which may appear superficially compatible with V3's nearly flat values ($1.355 \to 1.318$), but even these predict \emph{increasing}~$C_L$---whereas all three experiments show \emph{decreasing}~$C_L$. We therefore reformulate the question as an open problem (O2) about what mechanism determines $C_L$'s decay rate. Note: V2 stacks single-projection V0 blocks (not the full two-projection architecture of Theorem~\ref{thm:core}); nevertheless, the refutation is robust because Exp~B tests trained TinyLlama layers with equivalent activation function and still shows $C_L$ decline.

\begin{table}[!ht]
\centering
\caption{Multi-layer experiments: $C_L$ change from $L{=}1$ to $L{=}6$ across three architectures. All three refute exponential compounding.}
\label{tab:multilayer}
\begin{tabular}{@{}llccl@{}}
\toprule
Experiment & Architecture & $C_1$ & $C_6$ & Change \\
\midrule
V2 (i.i.d., no residual) & $[\ReLU(W_l\,\cdot) \to \text{RMSNorm}]^L$ & 3.587 & 3.066 & $-14.5\%$ \\
V3 (i.i.d., with residual) & $x + \text{FFN}(\xhat)$ & 1.355 & 1.318 & $-2.7\%$ \\
Exp B (trained, ReLU) & TinyLlama layers 8--13 & 3.494 & 2.966 & $-15.1\%$ \\
\bottomrule
\end{tabular}
\end{table}

\begin{figure}[ht]
\centering
\includegraphics[width=0.85\textwidth]{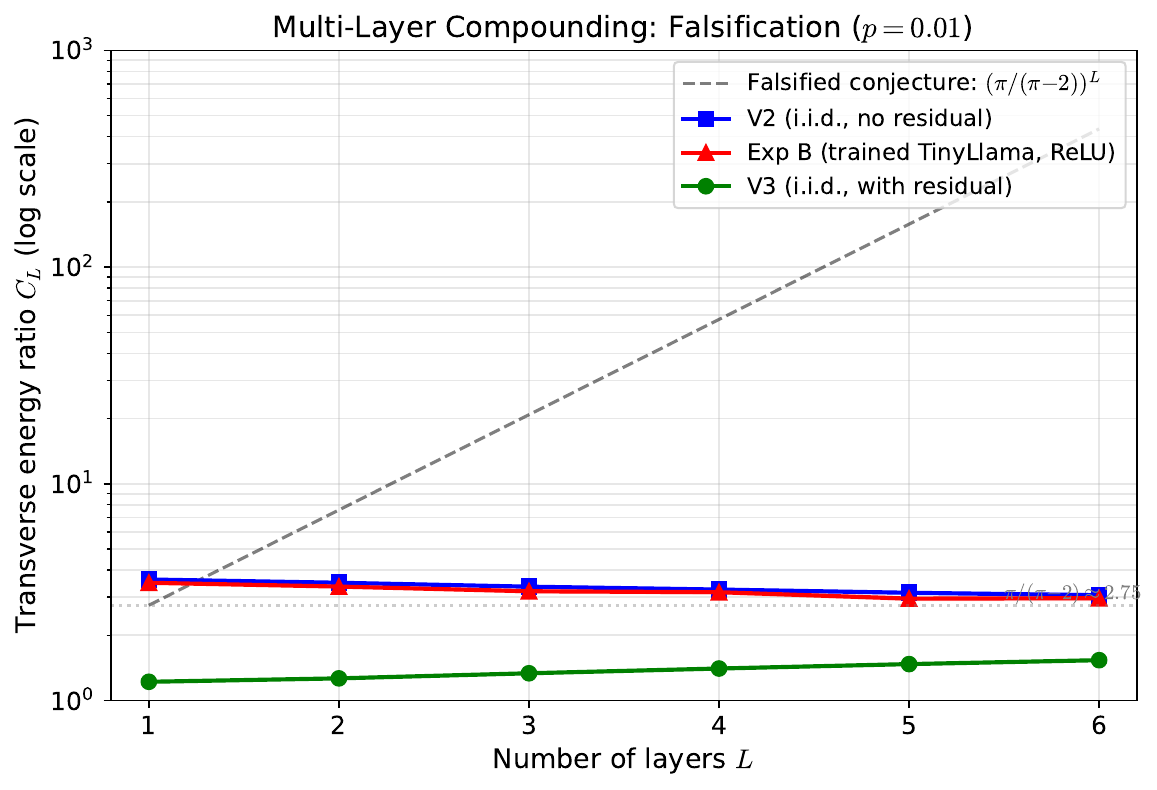}
\caption{Multi-layer compounding refutation ($p{=}0.01$). Dashed line: the refuted hypothesis $C_L = (\pi/(\pi{-}2))^L$, which predicts exponential growth to ${\sim}437$ at $L{=}6$. All three measured architectures---V2 (i.i.d., no residual), Exp~B (trained TinyLlama, ReLU), and V3 (i.i.d., with residual)---show flat or declining $C_L$, with observed $C_6 \in [1.3, 3.1]$. Shaded regions: 95\% CI.}
\label{fig:multilayer}
\end{figure}

\paragraph{Why $C_L$ declines.} After multiple layers, the intermediate activations $\xhat_l$ violate the assumptions of Theorem~\ref{thm:core}: (i)~ReLU produces half-Gaussian distributions (violating~(A1)); (ii)~cross-layer correlations accumulate, degrading isotropy; (iii)~the Bussgang constant $\sqrt{2/\pi}$ strictly requires Gaussian input.

\paragraph{Residual connection dilution.} Comparing V2 and V3 at $L=1$: residual connections reduce $C$ from 3.587 to 1.355 (a 62\% decrease), because the skip path $x_l$ provides an unperturbed fall-back that dilutes the sign/magnitude differential.

\subsection{Outlier Features Amplify Absolute Sign Sensitivity}

The gap from the single-layer baseline $2.75\times$ to the larger sign sensitivity observed in real models is not multi-layer compounding but rather \textbf{violation of assumption~(A2)}---the delocalization condition---by outlier features. \emph{Note}: The single-layer outlier leverage ratio $R(\alpha,n) \approx n\alpha^2$ is rigorously established (Appendix~\ref{app:outlier}, Theorem~\ref{thm:outlier-leverage}), and experimentally validated: at linear response ($p \leq 0.5\%$; Exp~E), count-matched NLL leverage stabilizes at ${\sim}10\times$, matching $n \cdot \mathbb{E}_{\mathrm{emp}}[\alpha^2] \approx 7$--$10$; Exps~F--G further confirm the $\alpha^2$ scaling (Spearman $\rho = 0.955$ and $0.927$). The column-group energy theory (Appendix~\ref{app:two-pop}) bounds $R_{\mathrm{col}} \leq 19$ for $\gamma \leq 0.5$. \emph{Log-loss resolution}: The all-column $\Delta$PPL ratio ($1{,}265\times$ median) appeared to exceed this by ${\sim}67\times$, but since $\mathrm{PPL} = \exp(\mathrm{NLL})$ is exponentially nonlinear, the correct comparison uses log-loss: the all-column NLL ratio is only $5.0\times$ (median, IQR $[4.8, 5.4]$), which falls \emph{within} $R_{\mathrm{col}} \leq 19$. \emph{Caveat}: $R_{\mathrm{col}}$ is a post-ReLU \emph{energy} ratio, while NLL is a downstream loss metric; their proportionality is a heuristic alignment (under small perturbations $\Delta\mathrm{NLL} \approx \frac{1}{2}(\Delta y)^\top H\,(\Delta y)$, which reduces to energy scaling when the loss Hessian $H$ is approximately isotropic in the perturbation subspace). A rigorous energy-to-NLL mapping is an open problem (O1c). The $251\times$ compression from PPL to NLL ratio accounts for the bulk of the apparent gap. Multi-layer propagation (O1b) remains open but is no longer needed to explain the all-column observation.

\emph{Important distinction}: Theorem~\ref{thm:core}'s $\pi/(\pi-2)$ quantifies the ratio of \emph{sign-flip vs.\ magnitude} perturbation damage (cross-type comparison, equal Frobenius norm). The Exp~D measurement in Section~9.5 quantifies the ratio of \emph{outlier-column vs.\ non-outlier-column} sign-flip damage (within-type comparison, same perturbation type applied to different columns). These are complementary observations: Exp~D demonstrates that \emph{absolute} sign sensitivity is dominated by outlier features (explaining why real models show extreme sign sensitivity overall), while Theorem~\ref{thm:core} explains the \emph{relative} mechanism by which sign perturbations are more damaging than magnitude perturbations in the delocalized baseline. Whether the sign-vs-magnitude ratio itself is amplified by outlier features remains an open question (O1).

\paragraph{Mechanism.} Consider a sign-flip restricted to a single column~$k$ of $W$ (i.e., $\Delta W_{:,k} \neq 0$, $\Delta W_{:,j} = 0$ for $j \neq k$). Since $\Delta y = \Delta W \xhat = \Delta W_{:,k} \xhat_k$ (a scalar times an $m$-dimensional column vector):
\[
\|\Delta y\|^2 = \|\Delta W_{:,k}\|^2 \cdot \xhat_k^2
\]
For the general case of i.i.d.\ Bernoulli($p$) sign-flips across all entries:
\[
\E[\|\Delta y\|^2] = 4p \sum_k \xhat_k^2 \|W_{:,k}\|^2
\]
revealing that each column~$k$'s contribution is weighted by $\xhat_k^2$. Under delocalization ($\xhat_k^2 \approx 1/n$), all columns contribute equally. But real transformer hidden states contain \textbf{outlier features} \citep{dettmers2022}: a small fraction of dimensions with $\xhat_k^2 = O(1)$ rather than $O(1/n)$. Sign-flips on these columns have leverage $R \approx n\alpha^2$ times that of average columns (Theorem~\ref{thm:outlier-leverage}).

Magnitude perturbations on outlier columns are similarly amplified in total energy, but their geometric alignment---which under delocalization is characterized by the Bussgang constant $\sqrt{2/\pi}$ (Lemma~\ref{lem:bussgang-relu})---is conjectured to channel the amplification predominantly into the radial direction (absorbed by RMSNorm). Since Lemma~\ref{lem:bussgang-relu} relies on assumption~(A2), this radial-channeling argument is heuristic for outlier columns (see Section~10.3, limitation~4). Sign perturbations lack this radial shelter regardless, so outlier features may selectively amplify the sign/magnitude asymmetry beyond $2.75\times$---but this conjecture has not been formally verified (O1 in Section~\ref{sec:open}).

\subsection{Experimental Evidence: Outlier-Targeted Sign-Flip}

\textbf{Exp~D} tests whether outlier features are indeed the dominant amplifier. On TinyLlama-1.1B with 5\% sign-flip rate, 20 seeds.

\emph{Outlier identification protocol}: We define outlier columns as those in the top 5\% by activation magnitude. Specifically, we compute $|\xhat_k|$ (the absolute value of the normalized hidden state) averaged over 128 random prompts from the WikiText-2 validation set, measured at the FFN input of layer~12 (a mid-network layer). Columns are ranked by this mean activation magnitude; the top 5\% (by count) are designated ``outlier.'' Results are qualitatively stable across layers 8--16 and prompt subsets.

\begin{table}[!ht]
\centering
\caption{Outlier sign-flip experiment (Exp~D) on TinyLlama-1.1B (20 seeds). Outlier columns (top 5\% by activation magnitude) produce ${\sim}1{,}265\times$ (median) more perplexity damage than non-outlier columns. Note: the $\Delta$PPL ratio is computed per-seed then medianed; median-of-ratios differs from the ratio of the median PPL values shown.}
\label{tab:outlier}
\begin{tabular}{@{}lrcrc@{}}
\toprule
Strategy & Flipped weights & PPL (median) & Rel.\ to baseline & $\Delta$PPL ratio \\
\midrule
Baseline (no flip) & 0 & 7.38 & $1\times$ & --- \\
Non-outlier columns only & ${\sim}$46M & 41.1 & $5.6\times$ & $1\times$ (ref.) \\
Random (all columns) & ${\sim}$48.7M & 979 & $133\times$ & $29\times$ \\
Outlier columns only (top 5\%) & ${\sim}$2.41M & \textbf{41{,}727} & $\mathbf{5{,}654\times}$ & $\mathbf{1{,}265\times}$ (med.) \\
\bottomrule
\end{tabular}
\end{table}

The $\Delta\text{PPL}$ ratio (outlier vs.\ non-outlier) is $\mathbf{1{,}265\times}$ (median across 20 seeds; IQR: $[795, 1{,}629]$) and $1{,}777\times$ (mean; inflated by extreme seeds including seed~42 with ratio $7{,}128\times$). To isolate per-flip sensitivity from the 19:1 count confound, we run a \textbf{count-matched control}: flip the same number of weights (${\sim}2.4$M) in outlier vs.\ non-outlier columns, with perturbation norms matched to within 4\%. The count-matched per-flip $\Delta$PPL leverage is $\mathbf{77\times}$ (median; IQR: $[48, 97]$).\footnote{The count-matched leverage directly measures per-flip sensitivity without assuming linear $\Delta$PPL scaling. The na\"ive product $1{,}265 \times 19 \approx 24{,}000$ overestimates per-flip leverage by ${\sim}300\times$ because PPL${}= \exp(-\text{avg.\ log-likelihood})$ is highly nonlinear at the large-perturbation scale of the all-column experiment.} Direct measurement of $|\xhat_k|$ at layer~12 reveals outlier $\alpha$: median $0.024$, P90 $= 0.12$, max $= 0.26$ (Appendix~\ref{app:outlier}). Theorem~\ref{thm:outlier-leverage} predicts per-entry $R \in [20, 162]$ for the extreme tail ($\alpha \in [0.1, 0.3]$, ${\sim}16\%$ of outlier dimensions); the aggregate prediction is $n \cdot \mathbb{E}_{\mathrm{emp}}[\alpha^2] \approx 7$--$10$. The noise-floor sweep (Exp~E, Table~\ref{tab:noise-floor}) confirms this: at $p \leq 0.5\%$ (linear-response regime), count-matched NLL leverage stabilizes at ${\sim}10\times$. The larger $37.6\times$ NLL at $p = 5\%$ (Table~\ref{tab:count-matched}) reflects an empirically observed ${\sim}4\times$ multi-flip interaction inflation, not single-entry physics.

\subsection{Revised Narrative}

\begin{itemize}
\item \textbf{Original narrative}: ``Single-layer $2.75\times$, multi-layer compounding to large factors in real models.''
\item \textbf{Corrected narrative}: ``Single-layer $2.75\times$ (delocalized baseline); outlier features amplify sign sensitivity by orders of magnitude via (A2) violation (Section~9.5); multi-layer compounding does not occur.''
\end{itemize}

The constant $\pi/(\pi-2)$ remains correct as the \textbf{delocalized baseline}---it quantifies the intrinsic geometric asymmetry arising from the sign-preserving structure + ReLU + RMSNorm pipeline. The additional amplification in real models is an input-geometry effect (outlier concentration), not a depth effect.

\section{Application: Ternary Quantization Error}

\subsection{Theorem Statement}

\begin{theorem}[Ternary Quantization Angular Error]\label{thm:ternary}
Under assumptions \textup{(A1')--(A4')}, as $m, n \to \infty$ with $m/n = O(1)$:
\begin{equation}
\cos^2\!\angle(W\xhat,\; W_T\xhat) \xrightarrow{P} \frac{2}{\pi}
\end{equation}
Equivalently, the angular error satisfies $\sin^2\!\angle \to 1 - 2/\pi \approx 0.363$. After RMSNorm:
\begin{equation}
\frac{\|\RMSNorm(W_T\xhat) - \RMSNorm(W\xhat)\|^2}{\|\RMSNorm(W\xhat)\|^2} \xrightarrow{P} 2\!\left(1 - \sqrt{\frac{2}{\pi}}\right) \approx 0.404
\end{equation}
\end{theorem}

\subsection{Proof}

\textbf{Step 1: Row-wise Bussgang analysis.} Using $\E_{\xhat}[\xhat_j\xhat_k] = \delta_{jk}/n$, the row-level correlation between $y_i = w_i^\top \xhat$ and $y_{T,i} = s_i \sign(w_i)^\top \xhat$ is:
\begin{equation}
\Corr(y_i, y_{T,i} \mid W) = \frac{\|w_i\|_1}{\sqrt{n}\|w_i\|} \xrightarrow{\mathrm{a.s.}} \sqrt{\frac{2}{\pi}}
\end{equation}
by Fact~\ref{fact:bussgang}.

\textbf{Step 2: Vector-level concentration} (sketch; full proof in Appendix~B). Three quadratic forms $\xhat^\top M_k \xhat$ (for $M_k \in \{W^\top W,\, W_T^\top W_T,\, W^\top W_T\}$) must concentrate around their means $\tr(M_k)/n$. By L\'{e}vy's spherical concentration inequality with Lipschitz constant $L = 2\|M_k\|_{\mathrm{op}} = O(1)$ (via the operator norm bound $\|W\|_{\mathrm{op}} \to 1+\sqrt{\gamma}$; see Appendix~B), each deviates by $O(n^{-1/2+\epsilon})$. The row-scale averages $\frac{1}{m}\sum_i s_i^2 \to 2/(\pi n)$ follow from a triangular-array Chebyshev argument (variance $O(1/(mn^2))$).

\textbf{Step 3: Combining.} The three concentrated quadratic forms yield:
\[
\cos^2\!\angle(y, y_T) \xrightarrow{P} \frac{\frac{1}{m}\sum_i s_i^2}{\|W\|_F^2/(mn)} = \frac{2/(\pi n)}{1/n} = \frac{2}{\pi}
\]
(Here $\frac{1}{m}\sum_i s_i^2 \to 2/(\pi n)$ is the per-row average of $s_i^2$, and $\|W\|_F^2/(mn) \to 1/n$ is the per-entry variance; the factor $m$ cancels between numerator and denominator.)
By Step~1, $\E_{\xhat}[y_i \cdot y_{T,i} \mid W] = s_i^2 > 0$ for each row, so $\E_{\xhat}[y \cdot y_T \mid W] = \sum_i s_i^2 \to 2m/(\pi n)$ which is $\Theta(1)$ when $m/n = \gamma$. The L\'{e}vy concentration gives fluctuations $O(1/\sqrt{n})$ around this mean (Appendix~\ref{app:thm4}). Since the mean is $\Theta(1)$ and bounded away from zero while fluctuations are $O(1/\sqrt{n})$, L\'{e}vy concentration gives $\Pr(y \cdot y_T < 0) \leq 2\exp(-\Omega(n))$, so $y \cdot y_T > 0$ with probability $1 - o(1)$. The positive root $\cos\theta \to \sqrt{2/\pi}$ therefore applies. The spherical chord-length formula gives the RMSNorm error: $2(1-\sqrt{2/\pi}) \approx 0.404$. \qed

\subsection{A Mechanism for Ternary Quantization Tolerance}

The explanatory chain connecting Theorem~\ref{thm:ternary} to ternary quantization tolerance:

\begin{enumerate}
\item \textbf{Theorem~\ref{thm:ternary}}: Ternary quantization error $\Delta W = W_T - W$ satisfies $\Delta w_{ij} = \sign(w_{ij})(s_i - |w_{ij}|)$---a \textbf{sign-preserving} perturbation. The quantized weight $w_{ij} + \Delta w_{ij} = s_i\sign(w_{ij})$ preserves the sign of the original weight, so no weight signs are flipped.

\item \textbf{Theorem~\ref{thm:core}}: For the constant-$\delta$ sign-preserving model ($\Delta w_{ij} = \sign(w_{ij})\delta$), the Bussgang-ReLU mechanism (Lemma~\ref{lem:bussgang-relu}) channels $2/\pi$ of the post-ReLU perturbation energy into the radial direction, yielding $(\pi-2)/\pi \approx 36\%$ transverse energy relative to sign-flip perturbations.

\item \textbf{Structural analogy and caveat}: Ternary quantization error has element-wise varying magnitudes $s_i - |w_{ij}|$ (not the constant~$\delta$ of Theorem~\ref{thm:core}'s model), and the sign-preserving property means $\sign(w_{ij} + \Delta w_{ij}) = \sign(w_{ij})$ (the \emph{quantized} weight preserves sign), not that $\Delta w_{ij}$ is proportional to $\sign(w_{ij})$---indeed when $|w_{ij}| > s_i$, the perturbation $\Delta w_{ij}$ opposes $\sign(w_{ij})$. Consequently, the pre-ReLU Bussgang correlation has \textbf{opposite sign}: $\Corr(z_i, \Delta z_i) \to -\sqrt{1-2/\pi} \approx -0.60$ for ternary error, versus $+\sqrt{2/\pi} \approx +0.80$ for Theorem~\ref{thm:core}'s constant-$\delta$ model.

\emph{Derivation of $\Corr(z_i, \Delta z_i) \to -\sqrt{1-2/\pi}$}: The ternary perturbation is $\Delta w_{ij} = \sign(w_{ij})(s_i - |w_{ij}|)$. With $\xhat$ independent of~$W$ and $\E[\xhat_j \xhat_k] = \delta_{jk}/n$:
\begin{align*}
\Cov(z_i, \Delta z_i) &= \tfrac{1}{n}\textstyle\sum_j \E[w_{ij}\cdot\sign(w_{ij})(s_i - |w_{ij}|)] = \tfrac{1}{n}\textstyle\sum_j \E[|w_{ij}|(s_i - |w_{ij}|)].
\end{align*}
By i.i.d.\ symmetry, each term equals $\E[|w_{i1}|\cdot s_i] - \E[w_{i1}^2]$. Since $s_i = \frac{1}{n}\sum_k |w_{ik}|$, we get $\E[|w_{i1}|s_i] = \frac{1}{n}\E[w_{i1}^2] + \frac{n-1}{n}(\E|w_{i1}|)^2 = \frac{1}{n^2} + \frac{2(n-1)}{\pi n^2}$. Subtracting $\E[w_{i1}^2] = 1/n$ gives $\Cov(z_i, \Delta z_i) \to (2/\pi - 1)/n$ as $n \to \infty$. Similarly, $\Var(\Delta z_i) \to (1 - 2/\pi)/n$ (from $\Var(|w_{ij}| - s_i) \to (1-2/\pi)/n$). With $\Var(z_i) = 1/n$:
\[
\Corr(z_i, \Delta z_i) = \frac{(2/\pi - 1)/n}{\sqrt{(1/n)\cdot(1-2/\pi)/n}} = \frac{2/\pi - 1}{\sqrt{1 - 2/\pi}} = -\sqrt{1 - 2/\pi}.
\]

Since the radial fraction depends on $\Corr^2$, this gives a \textbf{pre-ReLU} radial fraction $\mathcal{R}^{\mathrm{ternary}}_{\mathrm{pre}} \to 1-2/\pi \approx 0.363$ versus $\mathcal{R}^{\mathrm{mag}}_{\mathrm{pre}} \to 2/\pi \approx 0.637$. (These sum to~1 by the variance decomposition of $|w_{ij}|$: writing $|w_{ij}| = \E|w_{ij}| + (|w_{ij}| - \E|w_{ij}|)$ decomposes the absolute value into its mean---proportional to the constant-$\delta$ direction $\sign(w_{ij})\cdot\delta$---and a centered residual---proportional to the ternary direction $\sign(w_{ij})(s_i - |w_{ij}|)$. Since $w_{ij} = \sign(w_{ij})\cdot|w_{ij}|$, the identity $\E[w^2] = (\E|w|)^2 + \Var(|w|)$ gives $(\E|w|)^2/\E[w^2] + \Var(|w|)/\E[w^2] = 2/\pi + (1-2/\pi) = 1$. After projection through $\xhat$ (independent of $W$), the squared correlations of $z_i$ with each perturbation direction inherit this partition.) \emph{Important}: these are pre-ReLU radial fractions; the post-ReLU values are conditional on gate-preservation (i.e., $z_i + \Delta z_i$ keeping the same sign as $z_i$). For Theorem~\ref{thm:core}'s constant-$\delta$ model, $|\Delta z_i| = O(\delta)$ while $f_{z_i}(0) = \sqrt{n/(2\pi)}$, giving gate-flip probability $O(\sqrt{n}\delta) = O(\sqrt{p})$ (small for $p \to 0$); the pre- and post-ReLU fractions therefore agree to leading order. For ternary quantization, however, the perturbation $|\Delta z_i| = O(1/\sqrt{n})$---comparable to $|z_i|$ itself---so the gate-flip probability is $O(1)$, and the pre-ReLU radial fraction does \emph{not} directly transfer to the post-ReLU regime. Ternary quantization error is thus \emph{less} radial than the idealized constant-$\delta$ model at the pre-ReLU level, though still more radial than sign-flip perturbations ($\mathcal{R}^{\mathrm{sign}} \approx p$). The qualitative ordering $\mathcal{R}^{\mathrm{sign}}_{\mathrm{pre}} < \mathcal{R}^{\mathrm{ternary}}_{\mathrm{pre}} < \mathcal{R}^{\mathrm{mag(const.\text{-}\delta)}}_{\mathrm{pre}}$ holds at the pre-ReLU level; the post-ReLU radial fraction for ternary error is computed in Section~\ref{sec:post-relu-ternary} and Appendix~\ref{app:ternary-post}.
\end{enumerate}

\subsection{Post-ReLU Ternary Radial Fraction}\label{sec:post-relu-ternary}

The pre-ReLU analysis above leaves open whether the $O(1)$ gate-flip rate (${\sim}20.6\%$) significantly alters the radial fraction. We resolve this by computing the post-ReLU radial fraction directly (full derivation in Appendix~\ref{app:ternary-post}).

\paragraph{Setup.} The pre-activations $(z_i, z_i^Q)$ are asymptotically bivariate Gaussian (Lindeberg--Feller CLT under delocalization) with:
\[
\sigma_1 = 1/\sqrt{n}, \quad \sigma_2 = \sqrt{2/(\pi n)} \approx 0.798\,\sigma_1, \quad \rho = \sqrt{2/\pi}.
\]
Note $\sigma_2 < \sigma_1$: the quantized model produces systematically smaller pre-activations.

\paragraph{Gate-flip rate.} By the Van Vleck arcsine law, $P_{\mathrm{flip}} = \frac{1}{\pi}\arccos(\rho) = \phi/\pi \approx 20.6\%$, where $\phi = \arccos(\sqrt{2/\pi}) \approx 0.647\,\mathrm{rad}$.

\paragraph{Post-ReLU energy.} Using the Price--Cho--Saul formula $\E[\ReLU(X)\ReLU(Y)] = \frac{\sigma_X\sigma_Y}{2\pi}(\sin\theta + (\pi-\theta)\cos\theta)$ for bivariate normals with $\theta = \arccos(\rho)$, define the auxiliary constant:
\[
S = \frac{\cos\phi\,(\sin\phi + (\pi-\phi)\cos\phi)}{2\pi} \approx 0.3293.
\]
The per-neuron error energy is $\E[\delta a_i^2] = \frac{1}{n}(\tfrac{1}{2} + \tfrac{1}{\pi} - 2S) \approx 0.1597/n$.

\paragraph{Radial fraction.} The radial projection is $\langle \delta a, \hat{a}\rangle \to (S - \tfrac{1}{2})\sqrt{2} \approx -0.241$ (negative: ternary quantization contracts the activation vector). The post-ReLU radial fraction is:
\begin{equation}\label{eq:ternary-post-relu}
\mathcal{R}^{\mathrm{ternary}}_{\mathrm{post}} = \frac{2(S - \tfrac{1}{2})^2}{\tfrac{1}{2} + \tfrac{1}{\pi} - 2S} \approx 0.365.
\end{equation}

\paragraph{Near-equality with pre-ReLU.} The pre-ReLU value $\mathcal{R}^{\mathrm{ternary}}_{\mathrm{pre}} = 1 - 2/\pi \approx 0.363$ is \emph{exact} (Bussgang orthogonality: $\E[\delta z \cdot z^Q] = 0$). The post-ReLU value $0.365$ is the asymptotic closed form ($n \to \infty$). These differ by only ${\sim}0.4\%$ relative. Despite the 20.6\% gate-flip rate, ReLU is approximately ``transparent'' to the radial/transverse decomposition for ternary error---a non-trivial structural property absent from the small-$\delta$ perturbative analysis of Theorem~\ref{thm:core}.

\paragraph{Equal-norm comparison.} Matching Frobenius norms between ternary error ($\|\Delta W\|^2 = n(1-2/\pi)$) and sign-flip perturbation requires $p = (1-2/\pi)/4 \approx 0.091$. At this $p$, sign-flips produce $1.43\times$ more transverse energy than ternary error, consistent with $c(0.091) \approx 2.18$ from Theorem~\ref{thm:core}.

\begin{table}[!ht]
\centering
\caption{Post-ReLU radial fractions across perturbation types.}
\label{tab:perturbation-types}
\begin{tabular}{@{}lccc@{}}
\toprule
Perturbation & Pre-ReLU $\mathcal{R}$ & Post-ReLU $\mathcal{R}$ & Gate-flip rate \\
\midrule
Sign-flip (small $p$) & $p$ & $\approx p$ & $O(\sqrt{p})$ \\
Magnitude (const.-$\delta$, small $p$) & $2/\pi \approx 0.637$ & $2/\pi \approx 0.637$ & $O(\sqrt{p})$ \\
\textbf{Ternary quantization} & $\mathbf{1-2/\pi \approx 0.363}$ (exact) & $\mathbf{\approx 0.365}$ & $\mathbf{{\sim}20.6\%}$ \\
\bottomrule
\end{tabular}
\end{table}

\subsection{Scope and Caveats}

In the same angular metric, the pre-RMSNorm angular deviation is $1 - \cos\theta = 1 - \sqrt{2/\pi} \approx 0.202$, and the post-RMSNorm chord-length error is $\|\hat{u}_1 - \hat{u}_2\|^2 = 2(1 - \cos\theta) \approx 0.404$---simply twice the angular deviation (the standard relationship between angle and chord on a unit sphere). \textbf{Single-layer RMSNorm does not differentially suppress quantization error}: the radial fraction of ternary quantization error is $\mathcal{R} \approx 0.365$ both pre- and post-ReLU (Section~\ref{sec:post-relu-ternary}), i.e., only ${\sim}36.5\%$ radial and ${\sim}63.5\%$ transverse. Since the majority of the perturbation energy is already transverse, RMSNorm's radial filtering absorbs only a minority ($36.5\%$) of the error, while $63.5\%$ passes through as functionally relevant output distortion.

The true suppression mechanism requires the \textbf{multi-component pipeline}: the sign-preserving weight structure provides Bussgang correlation between perturbation and activation; ReLU propagates this into the hidden space, creating the differential treatment of sign vs.\ magnitude perturbations; and RMSNorm then exposes this differential. The quantization tolerance is a property of the \emph{architecture} (sign structure + ReLU + RMSNorm together), not of RMSNorm alone.

\section{Experiments}

\subsection{Setup}

All synthetic experiments use i.i.d.\ $N(0, 1/n)$ weight matrices with 20 random seeds (unless noted), 10,000 spherical samples per seed. Real-model experiments use TinyLlama-1.1B (1.1B parameters, SiLU activation) and the Qwen2.5 family \citep{qwen2024} (0.5B--3B). Perplexity is evaluated on WikiText-2. Code, experimental scripts, and raw data are available at \url{https://github.com/donglei628/sign-magnitude-pre-norm}.

\subsection{Fact~1 Verification: Cross-Model \texorpdfstring{$\cos^2$}{cos2}}

Table~\ref{tab:cross-model} reports row-wise $\cos^2(w, \sign(w))$ for randomly sampled weight matrices from five trained models spanning 0.5B--7.7B parameters. All fall within 5\% of the Gaussian limit $2/\pi$, with deviations consistent with heavy-tailed weight distributions.

\begin{table}[!h]
\centering
\begin{tabular}{@{}lcccc@{}}
\toprule
Model & Parameters & Matrices sampled & $\cos^2$ mean $\pm$ std & Dev.\ from $2/\pi$ \\
\midrule
Qwen2.5-0.5B & 494M & 20 & $0.618 \pm 0.034$ & $-2.9\%$ \\
TinyLlama-1.1B & 1.1B & 20 & $0.627 \pm 0.029$ & $-1.6\%$ \\
Qwen2.5-1.5B & 1.5B & 20 & $0.622 \pm 0.029$ & $-2.2\%$ \\
Qwen2.5-3B & 3.1B & 20 & $0.606 \pm 0.082$ & $-4.8\%$ \\
Qwen-7B \citep{bai2023qwen} (1st-gen, 2023) & 7.7B & 20 & $0.633 \pm 0.014$ & $-0.6\%$ \\
\bottomrule
\end{tabular}
\caption{All models fall within 5\% of $2/\pi = 0.637$, across two model families (Qwen2.5 series and original Qwen-7B) and TinyLlama. The consistent negative bias is compatible with heavy-tailed weight distributions (for which $(\E|w|)^2/\E[w^2] < 2/\pi$). The elevated standard deviation of Qwen2.5-3B ($\pm 0.082$) reflects a wider spread of per-matrix $\cos^2$ values across its 20 sampled layers, likely due to a subset of layers (e.g., early attention projections) with heavier-tailed or more structured weight distributions.}\label{tab:cross-model}
\end{table}

\subsection{Theorem~3 Verification: Architecture Comparison}

We compare two architectures: V0 is a single-layer ablation $\yhat_0 = \RMSNorm(\ReLU(W_1 \xhat))$ that omits the second projection $W_2$; V1 is the full two-layer model $\yhat_1 = \RMSNorm(W_2 \cdot \ReLU(W_1 \xhat))$ as described in Theorem~\ref{thm:core}. The V0 deviation ($+37.8\%$) confirms that the second projection $W_2$ is essential for the theoretical asymmetry to manifest correctly; V1 matches theory within 1--6\%.

\begin{table}[!ht]
\centering
\begin{tabular}{@{}lcccc@{}}
\toprule
Architecture & $p$ & Measured $C$ & Theory $c(p)$ & Deviation \\
\midrule
V0 (incorrect: no $W_2$) & 0.01 & 3.595 & 2.610 & $+37.8\%$ \\
\textbf{V1 (correct: two-layer)} & \textbf{0.01} & \textbf{2.639} & \textbf{2.610} & $\mathbf{+1.1\%}$ \\
\textbf{V1} & \textbf{0.02} & \textbf{2.583} & \textbf{2.538} & $\mathbf{+1.8\%}$ \\
\textbf{V1} & \textbf{0.05} & \textbf{2.465} & \textbf{2.378} & $\mathbf{+3.7\%}$ \\
\textbf{V1} & \textbf{0.10} & \textbf{2.299} & \textbf{2.178} & $\mathbf{+5.6\%}$ \\
\bottomrule
\end{tabular}
\caption{The V0 deviation ($+37.8\%$) is primarily consistent with the missing $W_2$ layer (alternative contributions may include the different input distribution). The correct V1 architecture matches theory within 1--6\%.}
\end{table}

\subsection{Multi-Layer Experiments}

\textbf{V2 (no residual, $p{=}0.01$)}: Each layer applies $\xhat \mapsto \RMSNorm(\ReLU(W_l\,\xhat))$ with a fresh i.i.d.\ Gaussian $W_l \in \R^{m \times n}$ (no second projection $W_2$; hence V2 at $L{=}1$ matches V0, not V1). $C_L$ declines from 3.587 ($L=1$) to 3.066 ($L=6$), a $-14.5\%$ drop.

\textbf{V3 (with residual, $p{=}0.01$)}: $C_L$ plateaus from 1.355 ($L=1$) to 1.318 ($L=6$), a $-2.7\%$ drop.

\textbf{Exp~B (trained weights, ReLU)}: We extract weight matrices from TinyLlama-1.1B (which natively uses SwiGLU) and measure the transverse energy ratio by passing perturbations through these matrices with a ReLU activation (replacing the native gate mechanism) to test the Theorem~3 predictions directly. $C_L$ declines from 3.494 to 2.966 ($-15.1\%$). We also measure with SiLU activation (closer to the native SwiGLU but without the gate), finding qualitatively identical behavior (decline from 3.576 to 3.292, $-7.9\%$). The single-layer $C_1 = 3.494$ exceeds the i.i.d.\ theoretical value~2.75 by 27\%. This gap has multiple potential sources: (a)~non-Gaussian weight distributions (heavy tails, outlier features); (b)~the ReLU proxy replacing the native SwiGLU gate; (c)~row-wise weight correlations absent in the i.i.d.\ model; (d)~non-spherical input distributions in trained models. The SiLU comparison ($C_1 = 3.576$, only $+2.4\%$ from ReLU's 3.494) indicates that (b) contributes $<5\%$ of the gap, leaving the remaining ${\sim}22\%$ attributable to (a), (c), or (d). Full ablation is deferred to future work (O1).

All three refute exponential growth. The hypothesis predicted $C_6 \approx 437$; observed values are $1.3$--$3.1$.

\subsection{Outlier Experiments (Exp~D)}\label{sec:outlier-exp}

TinyLlama-1.1B, 5\% sign-flip rate, 20 seeds:

\begin{table}[!ht]
\centering
\begin{tabular}{@{}lrccc@{}}
\toprule
Strategy & Weights flipped & PPL median & $\Delta$PPL ratio & NLL ratio \\
\midrule
Baseline (no flip) & 0 & 7.38 & --- & --- \\
Non-outlier columns & 46M & 41.1 & $1\times$ (ref.) & $1\times$ (ref.) \\
Random & 48.7M & 979 & $29\times$ & --- \\
Outlier columns (top 5\%) & 2.41M & 41,727 & $\mathbf{1{,}265\times}$ & $\mathbf{5.0\times}$ \\
\midrule
\multicolumn{5}{@{}l}{\textit{Count-matched control (2.41M flips, norms matched to 4\%):}} \\
\quad Outlier columns & 2.41M & 26.1 & $\mathbf{77\times}$ & $\mathbf{37.6\times}$ \\
\quad Non-outlier columns & 2.41M & 7.62 & $1\times$ (ref.) & $1\times$ (ref.) \\
\bottomrule
\end{tabular}
\caption{Count-matched per-flip leverage (20 seeds). The NLL ratio ($= \ln\mathrm{PPL}$ ratio) removes the exponential nonlinearity of PPL: the all-column NLL ratio of $5.0\times$ falls within $R_{\mathrm{col}} \leq 19$, whereas the PPL ratio $1{,}265\times$ overstates the gap by $251\times$. All ratios are computed per-seed then medianed (median-of-ratios $\neq$ ratio-of-medians). Measured outlier $\alpha$: median $0.024$, P90 $= 0.12$, max $= 0.26$ (layer~12, top~5\%).}\label{tab:count-matched}
\end{table}

\paragraph{Exp~E: Noise-floor sweep.}
To test whether count-matched leverage is stable across signal regimes, we sweep the global flip rate $p \in \{5\%, 0.5\%, 0.1\%, 0.05\%, 0.01\%\}$ (fraction of total $W_1$ entries flipped, with equal absolute counts in outlier vs.\ non-outlier columns) while holding the count-matched protocol fixed (10 seeds each).

\begin{table}[!ht]
\centering
\begin{tabular}{@{}rrccc@{}}
\toprule
$p$ & Flips & NLL lev.\ (med.) & PPL lev.\ (med.) & PPL/NLL \\
\midrule
$5.00\%$ & 2{,}414K & $41.3\times$ & $89.8\times$ & $2.2\times$ \\
$0.50\%$ & 241K & $10.4\times$ & $10.5\times$ & $1.0\times$ \\
$0.10\%$ & 48K & $8.7\times$ & $8.7\times$ & $1.0\times$ \\
$0.05\%$ & 24K & $9.5\times$ & $9.5\times$ & $1.0\times$ \\
$0.01\%$ & 5K & $3.5\times$ & $2.3\times$ & noise \\
\bottomrule
\end{tabular}
\caption{Noise-floor sweep: count-matched NLL leverage vs.\ flip rate (10 seeds per $p$; independent seed set from Table~\ref{tab:count-matched}'s 20 seeds, explaining the $41.3\times$ vs.\ $37.6\times$ difference at $p = 5\%$). NLL leverage stabilizes at ${\sim}10\times$ for $p \leq 0.5\%$, consistent with the per-entry $R(\alpha)/R(\beta)$ ratio. The elevated $41.3\times$ at $p = 5\%$ reflects multi-flip interactions; at $p = 0.01\%$ the signal approaches the PPL measurement noise floor.}
\label{tab:noise-floor}
\end{table}

Two observations: (i)~NLL leverage stabilizes at ${\sim}10\times$ for $p \leq 0.5\%$, consistent with the per-entry ratio $R(\alpha)/R(\beta)$ (Theorem~\ref{thm:outlier-leverage}); the inflated $41.3\times$ at $p = 5\%$ reflects multi-flip interactions in outlier columns, not single-entry physics. (ii)~PPL and NLL leverage diverge only at $p = 5\%$ (PPL/NLL $= 2.2\times$), confirming that the exponential nonlinearity of the perplexity metric inflates leverage at high flip rates. We identify $p \leq 0.5\%$ as the empirical linear-response threshold---the regime where NLL leverage converges to a stable value. This threshold is determined from the data (not from a prior theoretical criterion); a rigorous characterization of the transition from single-flip to multi-flip behavior remains open.

\paragraph{Exp~F: Activation-level perturbation energy.}
Single-entry sign-flips produce $\Delta\mathrm{NLL} \sim 10^{-5}$, below the PPL measurement precision (${\sim}\pm 10^{-4}$). To bypass this noise floor, we directly compute the layer-output perturbation energy $\|\Delta y\|^2 = 4 w_{ij}^2 x_j^2$ for every input dimension~$j$ of the target layer (layer~12, \texttt{q\_proj}, $2048 \times 2048$). Since $\E[\|\Delta y\|^2] \propto \|W_{:,j}\|^2 \cdot \E[x_j^2]$ and $\E[x_j^2] = \E[\|x\|^2] \cdot \alpha_j^2$, the perturbation energy should scale as $\alpha_j^2$ when weight column norms are approximately uniform.

Results (10 sequences, 20{,}480 tokens): Spearman$(\alpha^2, \text{energy}) = 0.955$ ($p \approx 0$); energy ratio (top-5\% vs.\ rest) $= 33\times$; weight column norm ratio $= 1.04\times$ (approximately uniform). The near-unity weight norm ratio confirms that the $33\times$ energy gap is driven almost entirely by $\alpha^2$ variation, consistent with the $R \propto n\alpha^2$ prediction.

\paragraph{Exp~G: Column-flip $\Delta$NLL verification.}
To measure $\alpha^2$ scaling directly in $\Delta$NLL, we flip entire columns: $W_{:,j} \to -W_{:,j}$ for a single column~$j$, amplifying the signal by ${\sim}m$ (number of rows). We select 20 columns via log-spaced sampling across the full $\alpha$ distribution and measure $\Delta\mathrm{NLL} = \ln(\mathrm{PPL}_{\text{flip}}) - \ln(\mathrm{PPL}_{\text{base}})$.

Results (17/20 columns with positive $\Delta$NLL; the remaining 3 yielded $|\Delta\mathrm{NLL}| < 10^{-4}$, below per-token PPL precision): Spearman$(\alpha^2, \Delta\text{NLL}) = 0.927$ ($p = 9.3 \times 10^{-8}$). Log-log regression yields $\Delta\mathrm{NLL} \propto (\alpha^2)^{0.78}$, i.e., effective scaling $\sim\!\alpha^{1.57}$ (theory predicts $\alpha^2$). The sub-quadratic exponent is consistent with downstream multi-layer processing compressing large perturbations---the same attenuation that reduces NLL leverage from $41\times$ at $p = 5\%$ to ${\sim}10\times$ at $p \leq 0.5\%$ in Exp~E. The strong rank correlation ($\rho = 0.93$) confirms that $\alpha^2$ correctly predicts the \emph{ordering} of per-column damage.

\subsection{Theorem~4 Verification (Synthetic) and Step~1 Verification (Real Models)}

\paragraph{Synthetic} ($N(0,1/n)$ weights, uniform spherical input):

\begin{table}[!ht]
\centering
\caption{Theorem~4 synthetic verification: $\cos^2(y, y_T)$ converges to $2/\pi$ as dimension grows.}
\label{tab:thm4-synthetic}
\begin{tabular}{@{}lccc@{}}
\toprule
Dimensions $(m \times n)$ & $\cos^2$ measured & Theory $2/\pi$ & Relative error \\
\midrule
$256 \times 512$ & 0.6372 & 0.6366 & $+0.09\%$ \\
$512 \times 1024$ & 0.6369 & 0.6366 & $+0.04\%$ \\
$1024 \times 2048$ & 0.6369 & 0.6366 & $+0.04\%$ \\
$2048 \times 4096$ & 0.6367 & 0.6366 & $+0.02\%$ \\
\bottomrule
\end{tabular}
\end{table}

\paragraph{Real model---row-level verification (Fact~1 / Theorem~4 Step~1)} (TinyLlama-1.1B, 20 sampled weight matrices; row-wise $\cos^2(w_i, \sign(w_i))$ computed directly on weights, independent of any input). Note: this tests the row-level Bussgang prediction $\cos^2 \to 2/\pi$ (Fact~\ref{fact:bussgang}, equivalently Theorem~4 Step~1), not the full vector-level concentration (Theorem~4 Steps~2--3), which requires uniform spherical input:
\begin{itemize}
\item Median $\cos^2$: 0.624 (deviation $-2.0\%$ from $2/\pi$); Mean: 0.608; Std: 0.052
\item 16/20 matrices have deviation $< 3.5\%$ from $2/\pi$
\item Outlier matrices (4/20, predominantly early-layer attention projections) show larger deviations (up to $-36\%$), pulling the mean below the median; attributable to heavy-tailed weight distributions violating~(A1')
\end{itemize}

\subsection{Proposition~1 Verification}

Synthetic data ($n = m = 256$, $p = 0.01$, 20 seeds):

\begin{table}[!ht]
\centering
\caption{Verifies the $O(1/n)$ magnitude radial fraction (Proposition~\ref{prop:magnitude-uninformative}) in the single linear layer setup (no ReLU, no $W_2$). The energy ratio $= 1.0$ is by construction (Assumption~(A4) matches Frobenius norms), serving as a sanity check rather than a non-trivial prediction. Radial fractions here differ from the two-layer ReLU model of Theorem~3 (see Table~\ref{tab:radial-sign} in Appendix~D). The measured $\mathcal{R}^{\mathrm{sign}}$ is the sample average of per-seed radial fractions (Monte Carlo estimate), which differs from the theoretical ``ratio of moments'' at finite $n$; the $-11.7\%$ deviation reflects both this distinction and residual variance from 20 seeds.}
\label{tab:prop1}
\begin{tabular}{@{}lccc@{}}
\toprule
Quantity & Measured & Theory & Deviation \\
\midrule
$\mathcal{R}^{\mathrm{mag}}$ (radial fraction) & 0.00391 & $1/n = 0.00391$ & $< 0.1\%$ \\
$\mathcal{R}^{\mathrm{sign}}$ (radial fraction) & 0.00883 & $\approx p = 0.01$ & $-11.7\%$ \\
sign/mag energy ratio & 1.0000 & 1.0 (by construction) & $< 0.01\%$ \\
\bottomrule
\end{tabular}
\end{table}

\section{Discussion}

\subsection{Three Roles of \texorpdfstring{$2/\pi$}{2/pi}}

The Bussgang constant $2/\pi$ appears in three distinct contexts (Table~\ref{tab:three-roles}).

\begin{table}[H]
\centering
\caption{Appearances of the Bussgang constant $2/\pi$ and its complement $1-2/\pi$ in this work.}
\label{tab:three-roles}
\begin{tabular}{@{}llll@{}}
\toprule
Context & Quantity & Value & Object \\
\midrule
Fact~\ref{fact:bussgang} & $\cos^2(w, \sign(w))$ & $2/\pi$ & Weight vector itself \\
Theorem~\ref{thm:core} (B4) & $\mathcal{R}(\Delta y^{\mathrm{mag}})$ & $2/\pi$ & Mag.\ perturbation's radial fraction \\
Theorem~\ref{thm:ternary} & $\cos^2(y, y_T)$ & $2/\pi$ & Quantized vs.\ original alignment \\
\midrule
\S\ref{sec:post-relu-ternary} & $\mathcal{R}^{\mathrm{ternary}}_{\mathrm{pre}}$ & $1-2/\pi$ & Ternary error's pre-ReLU radial fraction (exact) \\
\S\ref{sec:post-relu-ternary} & $\mathcal{R}^{\mathrm{ternary}}_{\mathrm{post}}$ & $\approx 0.365$ & Ternary error's post-ReLU radial fraction \\
\bottomrule
\end{tabular}
\end{table}

All three share the same mathematical source (Bussgang gain of sign quantization on Gaussian vectors) but act on different objects.

\subsection{Rigor Levels}

Table~\ref{tab:rigor} summarizes the rigor level of each main result.

\begin{table}[H]
\centering
\caption{Rigor levels for each main result. ``Full'' denotes a complete, self-contained proof; ``Asymptotic'' denotes a proof that is rigorous in the stated limit ($n \to \infty$) with explicit remainder bounds, relying on Gaussian assumptions and CLT-based approximations.}
\label{tab:rigor}
\small
\begin{tabular}{@{}llp{0.22\textwidth}p{0.18\textwidth}p{0.20\textwidth}@{}}
\toprule
Result & Rigor & Key technique & Assumptions & Depends on \\
\midrule
Fact~1 & Full & SLLN + CMT & iid sym., $0{<}\E[w^2]{<}\infty$ & --- \\
Prop.~2 & Full & Vitali + Chernoff + Dir. & iid sym., $\E[w^8]{<}\infty$ & --- \\
Thm.~1 & Full (det.) & Pure algebra & $y \neq 0$ & --- \\
Thm.~2 & Full (det.) & Calculus + lin.\ alg. & $y \neq 0$ & --- \\
Prop.~1 & Full & Spherical expectation & iid $N(0,1/n)$ & --- \\
Thm.~3 & Asymptotic & Price's theorem + SLLN + Lemma~\ref{lem:bussgang-relu} & (A1)--(A6); Gaussian CLT & Fact~1, Thm.~2, Price's theorem \\
Thm.~4 & Asymptotic & L\'{e}vy concentration + triangular Chebyshev & (A1')--(A4'); Gaussian CLT & Fact~1, L\'{e}vy, sub-Gaussian concentration \\
\S\ref{sec:post-relu-ternary} & Asymptotic & Price--Cho--Saul + SLLN & (A1)--(A2); Gaussian CLT & Thm.~4, Van Vleck \\
App.~\ref{app:outlier} & Full & Plackett integral + ReLU energy & (A1); single-entry flip & --- \\
App.~\ref{app:two-pop} $R_{\mathrm{col}}$ & Full & $f(\rho)$ via Price--Cho--Saul & (A1); two-population $\hat{x}$ & App.~\ref{app:outlier} \\
\bottomrule
\end{tabular}
\end{table}

\subsection{Limitations}

\begin{enumerate}
\item \textbf{Single-layer theory}: The $\pi/(\pi-2)$ result applies to a two-layer toy model, not the full transformer stack.
\item \textbf{Gaussian assumption}: Real trained weights are not i.i.d.\ Gaussian; the cross-model $\cos^2$ deviations from $2/\pi$ reported in Section~9.2 reflect this.
\item \textbf{ReLU only}: Theorem~\ref{thm:core} is proved for ReLU, which provides the clean half-space partition enabling closed-form integration. Extension to gated activations (SwiGLU/GeGLU, used in modern LLMs) requires a bivariate-Gaussian-conditional-on-gate analysis not attempted here (though Exp~B shows qualitatively similar $C_L$ behavior with SiLU).
\item \textbf{Outlier theory gap}: The single-layer outlier leverage ratio ($R \approx n\alpha^2$, Theorem~\ref{thm:outlier-leverage}) predicts per-entry $R \in [20, 162]$ for extreme outliers ($\alpha \in [0.1, 0.3]$, ${\sim}16\%$ of outlier dimensions); the aggregate $n \cdot \mathbb{E}_{\mathrm{emp}}[\alpha^2] \approx 7$--$10$ accounts for the full $\alpha$ distribution. At linear response ($p \leq 0.5\%$; Exp~E), count-matched NLL leverage stabilizes at ${\sim}10\times$, matching this prediction; the $37.6\times$ at $p = 5\%$ reflects an empirically observed ${\sim}4\times$ multi-flip interaction inflation (no closed-form derivation is available for the interaction terms). The all-column $\Delta$PPL ratio ($1{,}265\times$) appeared to exceed the energy prediction ($R_{\mathrm{col}} \leq 19$) by ${\sim}67\times$, but log-loss analysis shows this gap is an artifact of PPL nonlinearity: the all-column NLL ratio is only $5.0\times$ (Table~\ref{tab:count-matched}), well within $R_{\mathrm{col}} \leq 19$. Multi-layer propagation remains open but is no longer needed to explain the observation. The $\alpha^2$ scaling is additionally confirmed by noise-floor sweep (Exp~E, Table~\ref{tab:noise-floor}), activation-level perturbation energy (Exp~F, Spearman $\rho = 0.955$), and column-flip $\Delta$NLL measurements (Exp~G, Spearman $\rho = 0.927$).
\item \textbf{Multi-layer propagation}: How $\mathcal{R}$ evolves across depth is an open problem.
\item \textbf{Position dependence}: Theorem~\ref{thm:core} applies to weights whose output passes through a downstream ReLU (the $W_1$ position). Quantization of weights \emph{after} a nonlinearity (e.g., FFN down-projection in a real transformer) acts on non-Gaussian inputs, where the Bussgang relation (Lemma~\ref{lem:bussgang-relu}) no longer holds exactly. The mechanism may still apply approximately but is not covered by our analysis.
\end{enumerate}

\section{Conclusion}

\subsection{Summary}

We have provided a geometric explanation for why pre-norm transformers tolerate ternary weight quantization:

\begin{enumerate}
\item Sign patterns capture $2/\pi \approx 63.7\%$ of weight directional energy (Fact~\ref{fact:bussgang}), while magnitude with random signs is uninformative (Proposition~\ref{prop:magnitude-uninformative}).

\item RMSNorm acts as a radial filter whose Fr\'{e}chet derivative is a transverse projector (Theorem~\ref{thm:transverse-proj}), absorbing perturbations aligned with the output direction.

\item In a two-layer ReLU + RMSNorm model, sign-flip perturbations produce $\pi/(\pi-2) \approx 2.75$ times more transverse energy than magnitude perturbations (Theorem~\ref{thm:core}). The mechanism: the sign-preserving weight structure provides Bussgang alignment; ReLU propagates it into an asymmetry between sign-flip and magnitude perturbations' radial fractions ($\mathcal{R}^{\mathrm{sign}} \approx p$ vs.\ $\mathcal{R}^{\mathrm{mag}} \approx 2/\pi$); RMSNorm exposes this differential.

\item Ternary quantization error is a sign-preserving perturbation with angular alignment $\cos^2 \to 2/\pi$ (Theorem~\ref{thm:ternary}). Its pre-ReLU radial fraction ($1-2/\pi \approx 0.363$, exact via Bussgang orthogonality) places it between sign-flips ($\mathcal{R} \approx p$) and the idealized constant-$\delta$ model ($\mathcal{R} = 2/\pi$). Despite a 20.6\% gate-flip rate, the post-ReLU radial fraction ($\mathcal{R}^{\mathrm{ternary}}_{\mathrm{post}} \approx 0.365$; Section~\ref{sec:post-relu-ternary}, Appendix~\ref{app:ternary-post}) agrees with the pre-ReLU value to within $0.4\%$: ReLU is approximately transparent to the radial/transverse decomposition for ternary error.

\item The gap from the single-layer $2.75\times$ baseline to the substantial sign sensitivity observed in real models comes from outlier features violating delocalization (Section~7.2), not from multi-layer compounding (which is not experimentally supported). The per-entry leverage $R \approx n\alpha^2$ (Appendix~\ref{app:outlier}) and column-group leverage $R_{\mathrm{col}} \leq 19$ for $\gamma \leq 0.5$ (Appendix~\ref{app:two-pop}) are rigorously derived. At linear response ($p \leq 0.5\%$; Exp~E), count-matched NLL leverage stabilizes at ${\sim}10\times$, matching $n \cdot \mathbb{E}_{\mathrm{emp}}[\alpha^2] \approx 7$--$10$; Exps~F--G further confirm the $\alpha^2$ scaling (Spearman $\rho = 0.955$ and $0.927$, respectively).
\end{enumerate}

\subsection{Open Problems}\label{sec:open}

\begin{table}[!ht]
\centering
\caption{Open problems arising from this work.}
\label{tab:open-problems}
\begin{tabular}{@{}cp{0.32\textwidth}p{0.50\textwidth}@{}}
\toprule
\# & Problem & Status \\
\midrule
O1a & Single-layer outlier leverage ($R \approx n\alpha^2$, column-group $R_{\mathrm{col}}$) & Resolved (App.~\ref{app:outlier},~\ref{app:two-pop}) \\
O1b & Multi-layer compounding of outlier leverage & Open \\
O1c & Quantitative energy-to-PPL mapping & Partially resolved: at linear response ($p \leq 0.5\%$) NLL leverage ${\sim}10 \approx n\mathbb{E}[\alpha^2]$; multi-flip inflation at $p = 5\%$ and sub-quadratic $\alpha^{1.57}$ exponent (Exps~E, G) remain open \\
O2 & Single-layer closed form for SiLU/GELU activations & Open \\
O3 & Propagation law for $\mathcal{R}$ across depth in residual streams & Open \\
O4 & Cross-layer accumulation/decay of quantization $\cos^2$ & Open \\
O5 & Bussgang gain correction for non-Gaussian (trained) weights & Partial (\S9.2) \\
O6 & Error compensation in low-bit architectures (gain scaling) & Open \\
O7a & Post-ReLU ternary radial fraction & Resolved: $\mathcal{R} \approx 0.365$ (\S\ref{sec:post-relu-ternary}, App.~\ref{app:ternary-post}) \\
O7b & Gate-flip decomposition of post-ReLU ternary error & Open (conjectural; App.~\ref{app:ternary-post}) \\
\bottomrule
\end{tabular}
\end{table}

\FloatBarrier
\bibliographystyle{plainnat}
\bibliography{refs}

@book{vershynin2018high,
  author    = {Vershynin, Roman},
  title     = {High-Dimensional Probability: An Introduction with Applications in Data Science},
  publisher = {Cambridge University Press},
  year      = {2018}
}

@techreport{bussgang1952,
  author      = {Bussgang, J. J.},
  title       = {Crosscorrelation Functions of Amplitude-Distorted {G}aussian Signals},
  institution = {MIT Research Laboratory of Electronics},
  year        = {1952},
  number      = {216}
}

@book{ledoux2001concentration,
  author    = {Ledoux, Michel},
  title     = {The Concentration of Measure Phenomenon},
  series    = {Mathematical Surveys and Monographs},
  volume    = {89},
  publisher = {American Mathematical Society},
  year      = {2001}
}

@article{price1958useful,
  author  = {Price, Robert},
  title   = {A Useful Theorem for Nonlinear Devices Having {G}aussian Inputs},
  journal = {{IRE} Transactions on Information Theory},
  volume  = {4},
  number  = {2},
  pages   = {69--72},
  year    = {1958}
}

@misc{ba2016layer,
  author       = {Ba, Jimmy Lei and Kiros, Jamie Ryan and Hinton, Geoffrey E.},
  title        = {Layer Normalization},
  howpublished = {arXiv:1607.06450},
  year         = {2016}
}

@inproceedings{xiong2020layer,
  author    = {Xiong, Ruibin and Yang, Yunchang and He, Di and Zheng, Kai and Zheng, Shuxin and Xing, Chen and Zhang, Huishuai and Lan, Yanyan and Wang, Liwei and Liu, Tie-Yan},
  title     = {On Layer Normalization in the Transformer Architecture},
  booktitle = {International Conference on Machine Learning (ICML)},
  pages     = {10524--10533},
  year      = {2020},
  note      = {arXiv:2002.04745}
}

@inproceedings{zhang2019rmsnorm,
  author    = {Zhang, Biao and Sennrich, Rico},
  title     = {Root Mean Square Layer Normalization},
  booktitle = {Advances in Neural Information Processing Systems (NeurIPS)},
  volume    = {32},
  pages     = {12360--12371},
  year      = {2019},
  note      = {arXiv:1910.07467}
}

@inproceedings{courbariaux2015binaryconnect,
  author    = {Courbariaux, Matthieu and Bengio, Yoshua and David, Jean-Pierre},
  title     = {{BinaryConnect}: Training Deep Neural Networks with Binary Weights during Propagations},
  booktitle = {Advances in Neural Information Processing Systems (NeurIPS)},
  pages     = {3123--3131},
  year      = {2015},
  note      = {arXiv:1511.00363}
}

@inproceedings{dettmers2022,
  author    = {Dettmers, Tim and Lewis, Mike and Belkada, Younes and Zettlemoyer, Luke},
  title     = {{LLM.int8()}: 8-Bit Matrix Multiplication for Transformers at Scale},
  booktitle = {NeurIPS},
  year      = {2022}
}

@inproceedings{frantar2023gptq,
  author    = {Frantar, Elias and Ashkboos, Saleh and Hoefler, Torsten and Alistarh, Dan},
  title     = {{GPTQ}: Accurate Post-Training Quantization for Generative Pre-trained Transformers},
  booktitle = {International Conference on Learning Representations (ICLR)},
  year      = {2023},
  note      = {arXiv:2210.17323}
}

@article{hubara2018quantized,
  author    = {Hubara, Itay and Courbariaux, Matthieu and Soudry, Daniel and El-Yaniv, Ran and Bengio, Yoshua},
  title     = {Quantized Neural Networks: Training Neural Networks with Low Precision Weights and Activations},
  journal   = {Journal of Machine Learning Research},
  volume    = {18},
  number    = {187},
  pages     = {1--30},
  year      = {2018},
  note      = {arXiv:1609.07061}
}

@inproceedings{lin2024awq,
  author    = {Lin, Ji and Tang, Jiaming and Tang, Haotian and Yang, Shang and Chen, Wei-Ming and Wang, Wei-Chen and Xiao, Guangxuan and Dang, Xingyu and Gan, Chuang and Han, Song},
  title     = {{AWQ}: Activation-aware Weight Quantization for On-Device {LLM} Compression and Acceleration},
  booktitle = {Proceedings of Machine Learning and Systems (MLSys)},
  year      = {2024},
  note      = {arXiv:2306.00978}
}

@inproceedings{rastegari2016,
  author    = {Rastegari, Mohammad and Ordonez, Vicente and Redmon, Joseph and Farhadi, Ali},
  title     = {{XNOR-Net}: {ImageNet} Classification Using Binary Convolutional Neural Networks},
  booktitle = {ECCV},
  year      = {2016}
}

@inproceedings{xiao2023smoothquant,
  author    = {Xiao, Guangxuan and Lin, Ji and Seznec, Mickael and Wu, Hao and Demouth, Julien and Han, Song},
  title     = {{SmoothQuant}: Accurate and Efficient Post-Training Quantization for Large Language Models},
  booktitle = {Proceedings of the 40th International Conference on Machine Learning (ICML)},
  series    = {PMLR},
  volume    = {202},
  pages     = {38087--38099},
  year      = {2023},
  note      = {arXiv:2211.10438}
}

@misc{wang2023bitnet,
  author       = {Wang, Hongyu and Ma, Shuming and Dong, Li and Huang, Shaohan and Wang, Huaijie and Ma, Lingxiao and Yang, Fan and Wang, Ruiping and Wu, Yi and Wei, Furu},
  title        = {{BitNet}: Scaling 1-bit {T}ransformers for Large Language Models},
  howpublished = {arXiv:2310.11453},
  year         = {2023}
}

@misc{ma2024bitnet,
  author       = {Ma, Shuming and Wang, Hongyu and Ma, Lingxiao and Wang, Lei and
                  Wang, Wenhui and Huang, Shaohan and Dong, Li and Wang, Ruiping and
                  Xue, Jilong and Wei, Furu},
  title        = {The Era of 1-Bit {LLMs}: All Large Language Models are in 1.58 Bits},
  howpublished = {arXiv:2402.17764},
  year         = {2024}
}

@inproceedings{frankle2019lottery,
  author    = {Frankle, Jonathan and Carbin, Michael},
  title     = {The Lottery Ticket Hypothesis: Finding Sparse, Trainable Neural Networks},
  booktitle = {International Conference on Learning Representations (ICLR)},
  year      = {2019},
  note      = {arXiv:1803.03635}
}

@inproceedings{oh2025sign,
  author    = {Oh, Junghun and Baik, Sungyong and Lee, Kyoung Mu},
  title     = {Find A Winning Sign: Sign Is All We Need to Win the Lottery},
  booktitle = {ICLR},
  year      = {2025},
  note      = {arXiv:2504.05357}
}

@inproceedings{zhou2019signs,
  author    = {Zhou, Hattie and Lan, Janice and Liu, Rosanne and Yosinski, Jason},
  title     = {Deconstructing Lottery Tickets: Zeros, Signs, and the Supermask},
  booktitle = {NeurIPS},
  year      = {2019}
}

@misc{vardhan2026,
  author       = {Vardhan, Mangadoddi Srikar and Teja, Lekkala Sai},
  title        = {Disentangling Direction and Magnitude in Transformer Representations:
                  A Double Dissociation Through {L2}-Matched Perturbation Analysis},
  howpublished = {arXiv:2602.11169},
  year         = {2026},
  note         = {Preprint, publicly available}
}

@misc{qwen2024,
  author       = {Yang, An and Yang, Baosong and Hui, Binyuan and Zheng, Bo and Yu, Bowen and others},
  title        = {Qwen2 Technical Report},
  howpublished = {arXiv:2407.10671},
  year         = {2024}
}

@misc{bai2023qwen,
  author       = {Bai, Jinze and Bai, Shuai and Chu, Yunfei and Cui, Zeyu and Dang, Kai and others},
  title        = {Qwen Technical Report},
  howpublished = {arXiv:2309.16609},
  year         = {2023}
}

\appendix
\numberwithin{lemma}{section}

\section{Complete Proof of Theorem~3}\label{app:thm3}

\subsection{Part A: Energy Ratio Before RMSNorm}

\textbf{Goal}: Show $R_A(p) := \E[\|\Delta y^{\mathrm{sign}}\|^2]/\E[\|\Delta y^{\mathrm{mag}}\|^2] = 1 - 4\sqrt{p}/(3\pi) + O(\sqrt{p})$.

\paragraph{Step A1: Output energy via $W_2$.} Since $W_2 \in \R^{d_{\mathrm{out}} \times m}$ has i.i.d.\ $N(0,1/m)$ entries independent of $(W_1, \xhat)$:
\[
\E[\|\Delta y\|^2 \mid \Delta a] = \frac{d_{\mathrm{out}}}{m}\|\Delta a\|^2
\]
Taking expectations: $\E[\|\Delta y\|^2] = d_{\mathrm{out}} \cdot \E[(\Delta a_i)^2]$ (by i.i.d.\ structure of rows of $W_1$).

\paragraph{Step A2: Activation change decomposition.} For neuron $i$ with pre-activation $z_i = w_i^\top \xhat$ (original) and $z_i' = w_i'^\top \xhat$ (perturbed):
\[
\Delta a_i = \ReLU(z_i') - \ReLU(z_i)
\]
Decompose into gate-flip (sign of $z_i$ changes) and smooth (both remain active) contributions:
\begin{align}
\E[(\Delta a_i)^2] &= \underbrace{\E[z_i^2 \mathbf{1}_{z_i > 0,\, z_i' < 0}] + \E[z_i'^2 \mathbf{1}_{z_i < 0,\, z_i' > 0}]}_{2\beta_{\mathrm{flip}}} \notag \\
&\quad + \underbrace{\E[(z_i' - z_i)^2 \mathbf{1}_{z_i > 0,\, z_i' > 0}]}_{\beta_{\mathrm{smooth}}}
\end{align}
(Note: on $\{z_i < 0, z_i' < 0\}$ both ReLU outputs are zero, contributing nothing to $\E[(\Delta a_i)^2]$.)

\paragraph{Step A3: Rescaling to unit variance.}
For $w_i \in \R^n$ with i.i.d.\ $N(0, 1/n)$ entries and $\xhat \in S^{n-1}$, we have
\[
\Var(z_i) \;=\; \xhat^\top \E[w_i w_i^\top] \xhat \;=\; \frac{\|\xhat\|^2}{n} \;=\; \frac{1}{n}.
\]
Define the rescaled variables $X := \sqrt{n}\, z_i$ and $X' := \sqrt{n}\, z'_i$. Under assumption~(A1), both marginals $z_i$ and $z'_i$ are exactly Gaussian ($z_i$ is a linear combination of independent Gaussians; $z'_i = \sum_j S_{ij} W_{1,ij}\xhat_j$ with $S_{ij}W_{1,ij} \sim N(0,1/n)$ marginally). The correlation is exact:
\[
\rho \;=\; \Corr(z_i, z'_i) \;=\; \E[S_{ij}] \;=\; 1 - 2p,
\]
where $S_{ij} \in \{-1, +1\}$ denotes the sign-flip indicator with $\Pr(S_{ij} = -1) = p$. However, the \emph{joint} distribution $(z_i, z'_i)$ is not exactly bivariate Gaussian for finite~$n$: since each summand pair $(W_{ij}\xhat_j,\, S_{ij}W_{ij}\xhat_j)$ is supported on two lines (depending on $S_{ij}$), exact joint normality holds only in the $n \to \infty$ limit by the multivariate Lindeberg-Feller CLT (via Cram\'{e}r-Wold device applied to all linear combinations of the bivariate sum) under delocalization~(A2). The Lindeberg condition is verified by $\max_j \xhat_j^2 = \|\xhat\|_\infty^2 = o(1)$, ensuring that the bivariate Gaussian approximation error vanishes in the asymptotic limit and does not affect the leading-order constant $\pi/(\pi-2)$. Concretely: the Berry--Esseen-type error for the bivariate CLT is $O(\|\xhat\|_\infty) = O(1/\sqrt{n})$ (the delocalization rate). Since Theorem~\ref{thm:core} takes $n \to \infty$ for \emph{fixed}~$p$ (see the statement: ``for each fixed $p$, as $n,m,d_{\mathrm{out}} \to \infty$''), this CLT error $O(1/\sqrt{n})$ vanishes before the small-$p$ expansion is applied, and thus does not contaminate the $p^{3/2}$ coefficients computed below.\footnote{Strictly, the uniformity of CLT convergence in~$p$ (i.e., $\sup_{p \in (\epsilon, 1/2)} |\beta_{\mathrm{flip}}(p,n) - \beta_{\mathrm{flip}}(p,\infty)| \to 0$ for each $\epsilon > 0$) is an implicit assumption here; as $p \to 0$ the gate-flip event degenerates, and Berry--Esseen-type bounds may depend on~$p$ through the conditional density at the boundary. The limit interchange is rigorously justified for any fixed $p > 0$ (Remark~\ref{rem:double-limit}); near-zero behavior is assumed by continuity of $\beta_{\mathrm{flip}}$ in~$p$.} We proceed under this asymptotic bivariate normal model with correlation~$\rho$.

\paragraph{Step A4: Computing $\beta_{\mathrm{flip}}$.}
In the asymptotic limit where $(X, X')$ converges to standard bivariate normal with correlation~$\rho$ (Step~A3), we compute the gate-flip contribution:
\[
\beta_{\mathrm{flip}} \;:=\; \E\!\left[X^2 \,\mathbf{1}_{X > 0,\, X' < 0}\right]
\]
By direct integration over the quarter-plane $\{X > 0, X' < 0\}$:
\[
\beta_{\mathrm{flip}} \;=\; \frac{1}{4} - \frac{1}{2\pi} \cdot \frac{a}{1 + a^2} - \frac{1}{2\pi} \arctan(a),
\qquad a \;=\; \frac{\rho}{\sqrt{1 - \rho^2}}.
\]
With $\rho = 1 - 2p$, we have $\sqrt{1 - \rho^2} = 2\sqrt{p(1-p)}$ and $a = \rho/\sqrt{1-\rho^2} \approx 1/(2\sqrt{p})$ for small $p$. Expanding to next order (the $p^{3/2}$ coefficients are needed to verify the final constant):
\begin{align*}
a/(1+a^2) &= \rho\sqrt{1-\rho^2} = 2\sqrt{p} - 5p^{3/2} + O(p^{5/2}), \\
\arctan(a) &= \pi/2 - \arctan(1/a) = \pi/2 - 2\sqrt{p} - p^{3/2}/3 + O(p^{5/2}).
\end{align*}
Substituting: the constant terms cancel ($1/4 - \frac{1}{2\pi}\cdot\frac{\pi}{2} = 1/4 - 1/4 = 0$), the $\sqrt{p}$ terms cancel, and collecting $p^{3/2}$ coefficients gives $5/(2\pi) + 1/(6\pi) = 16/(6\pi) = 8/(3\pi)$:
\begin{equation}\label{eq:beta-flip}
\beta_{\mathrm{flip}} \;=\; \frac{8 p^{3/2}}{3\pi} + O(p^{5/2}).
\end{equation}

\paragraph{Step A5: Computing $\beta_{\mathrm{smooth}}$.}
Using the ReLU product formula (a classical result obtained from Price's theorem~\citep{price1958useful} by integrating $\partial \E[\ReLU(X)\ReLU(X')] / \partial\rho = \E[\mathbf{1}_{X>0}\mathbf{1}_{X'>0}]$):
\[
\E[\max(X, 0)\max(X', 0)] \;=\; \frac{1}{2\pi}\left[\sqrt{1 - \rho^2} + \rho\left(\frac{\pi}{2} + \arcsin \rho\right)\right].
\]
The smooth-event energy (on $\{X > 0, X' > 0\}$, the region where both neurons remain active) is
\[
\beta_{\mathrm{smooth}} \;:=\; \E[(X' - X)^2 \,\mathbf{1}_{X > 0,\, X' > 0}]
\;=\; \E[X^2 \mathbf{1}_{X > 0,\, X' > 0}] + \E[X'^2 \mathbf{1}_{X > 0,\, X' > 0}] - 2 \E[X X' \mathbf{1}_{X > 0, X' > 0}].
\]
By exchangeability of $(X, X')$ under joint Gaussian with equal marginals, the first two terms are equal. Using $\E[X^2 \mathbf{1}_{X > 0,\, X' > 0}] = \E[X^2 \mathbf{1}_{X > 0}] - \E[X^2 \mathbf{1}_{X > 0,\, X' < 0}] = 1/2 - \beta_{\mathrm{flip}}$ and the ReLU-product formula $\E[XX'\mathbf{1}_{X>0,X'>0}] = \E[\max(X,0)\max(X',0)]$, substituting $\rho = 1 - 2p$ and expanding:
\begin{equation}\label{eq:beta-smooth}
\beta_{\mathrm{smooth}} \;=\; 2p - \frac{8 p^{3/2}}{\pi} + O(p^{5/2}).
\end{equation}

\paragraph{Step A6: Total sign-flip energy.}
Since the calculations of $\beta_{\mathrm{flip}}, \beta_{\mathrm{smooth}}$ were done in the asymptotic bivariate normal model (rescaled unit-variance scale), the sign-flip activation energy is (asymptotically as $n \to \infty$):
\[
\E[(\Delta a_i^{\mathrm{sign}})^2] \;=\; \frac{1}{n}\left(2 \beta_{\mathrm{flip}} + \beta_{\mathrm{smooth}}\right)
\;=\; \frac{1}{n} \cdot 2p\!\left(1 - \frac{4\sqrt{p}}{3\pi}\right) + O(p^{5/2}/n),
\]
where the factor $1/n$ accounts for $\Var(z_i) = 1/n$ (Step~A3), and
\[
\frac{16 p^{3/2}}{3\pi} - \frac{8 p^{3/2}}{\pi} \;=\; \frac{16 - 24}{3\pi}\, p^{3/2} \;=\; -\frac{8 p^{3/2}}{3\pi},
\]
yielding the displayed form.

\paragraph{Step A7: Magnitude energy (gate-state decomposition).}
For the sign-preserving model $\Delta z_i^{\mathrm{mag}} = \delta \sum_j \sign(W_{1, ij}) \xhat_j$ with $\delta^2 = 4p/n$:
\[
\Var(\Delta z_i^{\mathrm{mag}}) \;=\; \delta^2 \sum_j \xhat_j^2 \cdot \Var(\sign(W_{1, ij})) \;=\; \delta^2 \;=\; \frac{4p}{n}.
\]
After ReLU, decompose into four gate-state events for the perturbed activation $a_i' = \ReLU(z_i + \Delta z_i^{\mathrm{mag}})$:
\begin{itemize}
\item \emph{Smooth} (gate stays on): $z_i > 0$ and $z_i + \Delta z_i^{\mathrm{mag}} > 0$. Contribution: $\Delta a_i^{\mathrm{mag}} = \Delta z_i^{\mathrm{mag}}$. The probability of being in the boundary strip $|z_i| < O(\delta)$ is $f_{z_i}(0)\cdot O(\delta) = \sqrt{n/(2\pi)}\cdot O(\sqrt{p/n}) = O(\sqrt{p})$, so the smooth-on probability is $\tfrac{1}{2} - O(\sqrt{p})$, contributing $(\tfrac{1}{2} - O(\sqrt{p}))\cdot \delta^2 = 2p/n - O(p^{3/2}/n)$.
\item \emph{Gate stays off}: $z_i < 0$ and $z_i + \Delta z_i^{\mathrm{mag}} < 0$. Contribution: $\Delta a_i^{\mathrm{mag}} = 0$.
\item \emph{Gate-on flip} ($\text{off}\to\text{on}$): $z_i < 0 < z_i + \Delta z_i^{\mathrm{mag}}$. Probability $= O(\sqrt{p})$. On this event $|\Delta a_i| = |z_i + \Delta z_i| \leq O(\delta)$, so $(\Delta a_i)^2 = O(p/n)$. Contribution: $O(\sqrt{p})\cdot O(p/n) = O(p^{3/2}/n)$.
\item \emph{Gate-off flip} ($\text{on}\to\text{off}$): $z_i > 0$ and $z_i + \Delta z_i^{\mathrm{mag}} < 0$. Then $a_i = z_i$, $a_i' = 0$, so $\Delta a_i = -z_i$ with $|z_i| \leq O(\delta)$. Probability $= O(\sqrt{p})$ by the same boundary argument. Contribution: $O(\sqrt{p})\cdot O(\delta^2) = O(p^{3/2}/n)$.
\end{itemize}
Combining smooth + both gate-flip events:
\[
\E[(\Delta a_i^{\mathrm{mag}})^2] \;=\; \frac{2p}{n} + O(p^{3/2}/n).
\]

\paragraph{Step A8: Energy ratio.}
Both energies are at the $1/n$ scale. The denominator carries an $O(p^{3/2}/n)$ boundary correction from Step~A7, so dividing:
\[
R_A(p) \;:=\; \frac{\E[(\Delta a_i^{\mathrm{sign}})^2]}{\E[(\Delta a_i^{\mathrm{mag}})^2]}
\;=\; \frac{(2p/n)\bigl(1 - 4\sqrt{p}/(3\pi)\bigr)}{(2p/n)(1 + O(\sqrt{p}))}
\;=\; 1 + O(\sqrt{p}). \qquad \square
\]
\emph{Note on the structure of $R_A(p)$}: The numerator's factor $(1 - 4\sqrt{p}/(3\pi))$ has an exact coefficient from Steps~A4--A6. The denominator contributes a multiplicative $(1 + O(\sqrt{p}))^{-1}$ from Step~A7's boundary correction. We therefore write $R_A(p) = (1 - 4\sqrt{p}/(3\pi))(1 + r(p))$ where $|r(p)| = O(\sqrt{p})$ with coefficient depending on the Step~A7 boundary constants. Since $(1 - 4\sqrt{p}/(3\pi))$ and $(1 + r(p))$ are \emph{both} $1 + O(\sqrt{p})$, the only rigorous additive conclusion is $R_A(p) = 1 + O(\sqrt{p})$; the multiplicative factored form preserves the exact structural coefficient $4/(3\pi)$ for comparison with numerics while isolating the undetermined boundary remainder in $r(p)$. This is the source of $\varepsilon(p)$ in Remark~\ref{rem:structural}.

\subsection{Part B: Radial Fraction Analysis}\label{app:radial}

\textbf{Goal}: Show $\mathcal{R}^{\mathrm{sign}} \to p - 4p^{3/2}/(3\pi)$ and $\mathcal{R}^{\mathrm{mag}} \to 2/\pi$.

\paragraph{Step B1: Setup.} By Theorem~\ref{thm:transverse-proj}, after RMSNorm the transverse energy is $(1-\mathcal{R})\|\Delta y\|^2$ where $\mathcal{R} = (\Delta y \cdot \yhat)^2/\|\Delta y\|^2 = \cos^2(\Delta y, y)$. (Note: $(\Delta y \cdot \yhat)^2/\|\Delta y\|^2 = (\Delta y \cdot y)^2/(\|\Delta y\|^2 \|y\|^2)$ since $\yhat = y/\|y\|$.) By the random projection lemma~(B2), $\cos^2(\Delta y, y) \to \cos^2(\Delta a, a)$ as $d_{\mathrm{out}} \to \infty$, so we need to compute $\cos^2(\Delta a, a)$ in the hidden layer.

\paragraph{Step B2: Angle-preservation lemma.} For fixed $u, v \in \R^m$ and $W_2 \in \R^{d \times m}$ with i.i.d.\ $N(0,1/m)$: $W_2 u$ and $W_2 v$ are jointly Gaussian with $\Cov(W_2 u, W_2 v) = (u^\top v/m) I_d$. By concentration of $\|W_2 u\|^2$ around $\|u\|^2$ and $(W_2 u)^\top(W_2 v)$ around $u^\top v$:
\[
\cos^2(W_2 u, W_2 v) \xrightarrow{P} \cos^2(u, v) \quad (d \to \infty)
\]
\emph{Application}: In our setting, $u = \Delta a$ and $v = a$ are functions of $(W_1, \xhat)$. Since $W_2$ is independent of $(W_1, \xhat)$ by assumption~(A5), we condition on $(W_1, \xhat)$---making $u, v$ fixed from $W_2$'s perspective---apply the lemma, then take expectations over $(W_1, \xhat)$ using the dominated convergence theorem (or SLLN concentration of $\|u\|^2, \|v\|^2, u^\top v$ around their population means).

\paragraph{Step B3: Sign radial fraction.} We need $\cos^2(\Delta a^{\mathrm{sign}}, a)$. Three ingredients:

(i)~\emph{Reflection symmetry identity}: Since $(z_i, z_i') \overset{d}{=} (-z_i, -z_i')$ under the sign-flip model, expanding $\E[(\Delta a_i)^2] = 2\E[a_i^2] - 2\E[\ReLU(z_i')\ReLU(z_i)]$ (using equal marginals) gives:
\[
\E[\Delta a_i \cdot a_i] = \E[\ReLU(z_i')\ReLU(z_i)] - \E[\ReLU(z_i)^2] = -\frac{1}{2}\E[(\Delta a_i)^2]
\]

(ii)~\emph{Activation second moment}: From Step~A3, $\Var(z_i) = 1/n$. By symmetry:
\[
\E[a_i^2] = \E[\ReLU(z_i)^2] = \E[z_i^2 \mathbf{1}_{z_i > 0}] = \frac{1}{2}\Var(z_i) = \frac{1}{2n}
\]

(iii)~\emph{Perturbation second moment}: From Step~A6, $\E[(\Delta a_i)^2] = \beta_{\mathrm{total}}/n$, where $\beta_{\mathrm{total}} := 2\beta_{\mathrm{flip}} + \beta_{\mathrm{smooth}} = 2p(1-4\sqrt{p}/(3\pi)) + O(p^{5/2})$ denotes the unit-variance-scale total energy.

Combining via SLLN concentration and the random projection lemma~(B2) (all $1/n$ factors cancel in the ratio):
\[
\mathcal{R}^{\mathrm{sign}} = \cos^2(\Delta a, a) = \frac{(\E[\Delta a_i \cdot a_i])^2}{\E[(\Delta a_i)^2] \cdot \E[a_i^2]} = \frac{(\beta_{\mathrm{total}}/(2n))^2}{(\beta_{\mathrm{total}}/n) \cdot (1/(2n))} = \frac{\beta_{\mathrm{total}}}{2} = p - \frac{4p^{3/2}}{3\pi} + O(p^{5/2})
\]

\emph{Note on terminology}: Throughout this appendix, ``$\Corr(X,Y)$'' denotes the \textbf{uncentered} correlation $\E[XY]/\sqrt{\E[X^2]\E[Y^2]}$ (equivalently, the cosine between $X$ and $Y$ viewed as elements of $L^2$). This equals the Pearson correlation only when $\E[X] = 0$ or $\E[Y]=0$. Here $\E[a_i] = \E[\ReLU(z_i)] > 0$, so the uncentered and Pearson correlations differ; the relevant quantity for the vector-level $\cos^2(\Delta a, a)$ is the uncentered version.

\begin{lemma}[Bussgang-ReLU correlation]\label{lem:bussgang-relu}
Under assumptions \textup{(A1)--(A2)}, let $z_i = W_{1,i}\xhat$, $\Delta z_i^{\mathrm{mag}} = \delta \cdot \sign(W_{1,i}) \cdot \xhat$ with $\delta^2 = 4p/n$, $a_i = \ReLU(z_i)$, and $\Delta a_i = \ReLU(z_i + \Delta z_i) - \ReLU(z_i)$. Then as $n \to \infty$:
\[
\Corr(\Delta a_i, a_i) \to \sqrt{2/\pi} + O(\sqrt{p}).
\]
The $O(\sqrt{p})$ residual arises from gate-flip boundary events ($\Pr(0 < z_i < |\Delta z_i|) = O(\sqrt{p})$) that contribute $O(\sqrt{p})$ relative corrections to $\E[(\Delta A)^2]$. For finite~$n$, the boundary term in $\E[A\,\Delta A]$ is $O(\delta^2\sqrt{p}) = O(p^{3/2}/n)$ (proof part~(2)), which relative to the main term $\tfrac{1}{2}\Cov(z_i,\Delta z_i) = \delta/(2\sqrt{n})\cdot\sqrt{2/\pi} = O(\sqrt{p}/n)$ gives an $O(p)$ relative correction; combined with the $O(\sqrt{p})$ correction from $\E[(\Delta A)^2]$, the dominant residual is $O(\sqrt{p})$.
\end{lemma}

\begin{proof}
\textbf{(1) Joint asymptotic normality.}
Write $z_i = \sum_j W_j \xhat_j$ and $\Delta z_i = \delta \sum_j \sign(W_j) \xhat_j$, sums of independent terms across $j$. The bivariate moments are:
\begin{align*}
\Var(z_i) &\;=\; \sum_j \xhat_j^2 \cdot \Var(W_j) \;=\; \frac{1}{n}, \\
\Var(\Delta z_i) &\;=\; \delta^2 \sum_j \xhat_j^2 \cdot \Var(\sign(W_j)) \;=\; \delta^2, \\
\Cov(z_i, \Delta z_i) &\;=\; \delta \sum_j \xhat_j^2 \cdot \E[W_j \sign(W_j)]
\;=\; \delta \cdot \E[|W_1|] \;=\; \delta \sqrt{\tfrac{2}{\pi n}},
\end{align*}
where the last line uses $\E[W_j \sign(W_j)] = \E[|W_j|] = \sqrt{2/(\pi n)}$ for $W_j \sim N(0, 1/n)$. Thus
\[
\Corr(z_i, \Delta z_i) \;=\; \frac{\delta \sqrt{2/(\pi n)}}{\sqrt{(1/n) \cdot \delta^2}} \;=\; \sqrt{\tfrac{2}{\pi}}\quad \text{(population second-moment identity, valid for all } n\text{)}.
\]
Note: this is an identity for the ratio of population moments (first- and second-order); the \emph{joint distribution} of $(z_i, \Delta z_i)$ is not bivariate Gaussian at finite~$n$. Under delocalization $\|\xhat\|_\infty \to 0$, each summand contributes $o(\sigma_n)$ in probability with $\sigma_n := \sqrt{\Var(z_i)} = 1/\sqrt{n}$, so the multivariate Lindeberg-Feller condition is satisfied (via Cram\'{e}r-Wold). Define rescaled variables $Z := z_i / \sigma_n = \sqrt{n}\, z_i$ and $\Delta Z := \Delta z_i / \delta$ (each unit-variance). Then $(Z, \Delta Z) \xrightarrow{d} N(0, \Sigma)$ as $n \to \infty$, where $\Sigma$ has off-diagonal correlation $\sqrt{2/\pi}$.

\textbf{(2) Post-ReLU correlation.}
Let $A := \ReLU(z_i)$ and $\Delta A := \ReLU(z_i + \Delta z_i) - \ReLU(z_i)$. We compute $\E[A \cdot \Delta A]$ by decomposing on the smooth event $\mathcal{E}_s := \{z_i > 0, \, z_i + \Delta z_i > 0\}$ and its complement.

\emph{Smooth contribution.} On $\mathcal{E}_s$, $A = z_i$ and $\Delta A = \Delta z_i$, so
\[
\E[A \Delta A \, \mathbf{1}_{\mathcal{E}_s}]
\;=\; \E[z_i \Delta z_i \, \mathbf{1}_{z_i > 0}] - \E[z_i \Delta z_i \, \mathbf{1}_{z_i > 0, \, z_i + \Delta z_i < 0}].
\]
The first term: by the standard identity for jointly Gaussian variables (applied in the $n \to \infty$ limit where Part~(1) establishes $(z_i, \Delta z_i) \xrightarrow{d} N(0,\Sigma)$; by the bivariate CLT and uniform integrability (Part~3), $\E[z_i \Delta z_i \mathbf{1}_{z_i>0}]$ converges to the Gaussian-limit value with error $O(\|\xhat\|_\infty) = o(1)$, absorbed into the $O(\sqrt{p})$ remainder),
\[
\E[z_i \Delta z_i \, \mathbf{1}_{z_i > 0}] \;=\; \frac{\Cov(z_i, \Delta z_i)}{\Var(z_i)} \cdot \E[z_i^2 \mathbf{1}_{z_i > 0}]
\;=\; \frac{\Cov(z_i, \Delta z_i)}{\Var(z_i)} \cdot \frac{\Var(z_i)}{2} \;=\; \frac{1}{2} \Cov(z_i, \Delta z_i).
\]
The second (boundary) term is bounded by $E[|z_i| \cdot |\Delta z_i| \cdot \mathbf{1}_{\mathcal{B}}]$ where $\mathcal{B} := \{0 < z_i < -\Delta z_i, \Delta z_i < 0\}$. Since $\Pr(\mathcal{B}) = f_{z_i}(0) \cdot O(\delta) = O(\sqrt{p})$ (the density factor $\sqrt{n/(2\pi)}$ combines with $\delta = 2\sqrt{p/n}$) and on $\mathcal{B}$ we have $|z_i| \leq |\Delta z_i| \leq O(\delta)$, the boundary term is $O(\delta^2 \sqrt{p}) = O(p^{3/2}/n)$.

\emph{Gate-flip contribution.} On $\mathcal{E}_s^c \cap \{z_i > 0\}$ (gate turns off), $\Delta A = -z_i$. The contribution is $-\E[z_i^2 \mathbf{1}_{z_i > 0, z_i + \Delta z_i < 0}]$, bounded by $\sigma_n^2 \cdot \Pr(\mathcal{B}) = O(p^{3/2}/n)$. On $\mathcal{E}_s^c \cap \{z_i < 0\}$ (gate turns on), $A = 0$ so contribution vanishes.

Combining: $\E[A \Delta A] = \frac{1}{2}\Cov(z_i, \Delta z_i) + O(p^{3/2}/n) = \frac{\delta}{2\sqrt{n}}\sqrt{2/\pi} + O(p^{3/2}/n)$.

\emph{Correlation.} Since $z_i$ is unperturbed (independent of $p$), $\E[A^2] = \E[z_i^2 \mathbf{1}_{z_i > 0}] = \frac{1}{2}\Var(z_i) = \frac{1}{2n}$ exactly. For $\E[(\Delta A)^2] = \frac{1}{2}\delta^2(1 + O(\sqrt{p}))$ (the $O(\sqrt{p})$ accounts for gate-flip boundary events of probability $O(\delta\sqrt{n}) = O(\sqrt{p})$; cf.\ Step~A7):
\[
\Corr(\Delta A, A) \;=\; \frac{\E[A \Delta A]}{\sqrt{\E[A^2] \E[(\Delta A)^2]}}
\;=\; \frac{\frac{\delta}{2\sqrt{n}}\sqrt{2/\pi}}{\frac{\delta}{2\sqrt{n}}\,(1 + O(\sqrt{p}))^{1/2}}
\;=\; \sqrt{\tfrac{2}{\pi}} + O(\sqrt{p}).
\]

\textbf{(3) Uniform integrability.}
Since ReLU is 1-Lipschitz, $|A| \leq |z_i|$ and $|\Delta A| \leq |\Delta z_i|$. For any $\epsilon > 0$, by Cauchy--Schwarz: $\E[|A_n \Delta A_n|^{1+\epsilon}] \leq \E[|z_i|^{2(1+\epsilon)}]^{1/2}\,\E[|\Delta z_i|^{2(1+\epsilon)}]^{1/2}$, and both factors are uniformly bounded in~$n$ by the sub-Gaussian moment bound $\E[|X|^k] \leq (C\sigma)^k k^{k/2}$ with $\sigma^2 = \Var(z_i) = 1/n$. Hence $\{A_n \Delta A_n\}$ and $\{A_n^2\}$ are uniformly integrable, and moment convergence follows from the bivariate CLT.
\end{proof}

\paragraph{Step B4: Magnitude radial fraction.} Applying Lemma~\ref{lem:bussgang-relu}:
\[
\mathcal{R}^{\mathrm{mag}} = \cos^2(\Delta a, a) = \frac{(\E[\Delta a_i \cdot a_i])^2}{\E[(\Delta a_i)^2]\,\E[a_i^2]} \to 2/\pi + O(\sqrt{p})
\]
\subsection{Final Assembly}

\begin{align}\label{eq:final-assembly}
c(p) &= \frac{1-\mathcal{R}^{\mathrm{sign}}}{1-\mathcal{R}^{\mathrm{mag}}} \cdot R_A = \frac{1-p+4p^{3/2}/(3\pi)}{1-2/\pi} \cdot \left(1-\frac{4\sqrt{p}}{3\pi}\right)
\end{align}
This expression mixes different approximation orders: $\mathcal{R}^{\mathrm{sign}}$ is exact through order $p^{3/2}$, $R_A$ retains only the leading structural factor $(1-4\sqrt{p}/(3\pi))$ with an $O(\sqrt{p})$ multiplicative remainder, and $\mathcal{R}^{\mathrm{mag}}$ is fixed at its leading value $2/\pi$ (dropping the $O(\sqrt{p})$ gate-flip boundary correction from Lemma~\ref{lem:bussgang-relu}). The non-uniform precision means this is \emph{not} a consistent asymptotic expansion to any single order; rather, it is a structural decomposition that retains exactly computable coefficients while collecting undetermined remainders into $\varepsilon(p)$. Comparing with the structural decomposition (Remark~\ref{rem:structural}): $\varepsilon(p)$ absorbs (i)~the $O(\sqrt{p})$ correction to $\mathcal{R}^{\mathrm{mag}}$ in the denominator, (ii)~the $O(\sqrt{p})$ remainder in $R_A$ from Step~A8, and (iii)~the $O(p^{3/2})$ difference between the factored form $(1-p)$ and the exact numerator $(1-p+4p^{3/2}/(3\pi))$. The overall bound $|\varepsilon(p)| = O(\sqrt{p})$ follows from (i) and (ii) being $O(\sqrt{p})$; the coefficient and sign depend on uncalculated boundary constants in Steps~A7--A8. Empirically, $\varepsilon(p) > 0$ at all tested values (Table~2), but quantitative attribution is not attempted.

As $p \to 0$: numerator $\to 1$, denominator $(\pi-2)/\pi$, $R_A \to 1$. Thus $c(0) = \pi/(\pi-2)$.

For finite $p$, the leading correction:
\[
c(p) = \frac{\pi}{\pi-2}(1-p)\!\left(1-\frac{4\sqrt{p}}{3\pi}\right)\!\bigl(1 + \varepsilon(p)\bigr), \quad |\varepsilon(p)| = O(\sqrt{p}) \qquad \square
\]

\section{Complete Proof of Theorem~4, Step 2}\label{app:thm4}

\subsection{L\'{e}vy Spherical Concentration}

For $\xhat \sim \Unif(S^{n-1})$ and $f: S^{n-1} \to \R$ with Lipschitz constant $L$ (w.r.t.\ geodesic metric):
\begin{equation}
\Pr(|f(\xhat) - \E f| > t) \leq 2\exp\!\left(-\frac{(n-1)t^2}{2L^2}\right)
\end{equation}

For quadratic forms $f_M(\xhat) = \xhat^\top M \xhat$: $\|\nabla f_M(\xhat)\| = \|2M\xhat\| \leq 2\|M\|_{\mathrm{op}}$ on $S^{n-1}$, so $L = 2\|M\|_{\mathrm{op}}$.

\paragraph{Operator norm bounds (\citealt{vershynin2018high}, Theorem~4.6.1).} For $W \in \R^{m \times n}$ with i.i.d.\ $N(0,1/n)$ entries and $m/n \to \gamma$: $\|W\|_{\mathrm{op}} \to 1 + \sqrt{\gamma}$ a.s. (This is a classical result; for Gaussian matrices, it follows from the Marchenko--Pastur law; for sub-Gaussian matrices, see \citealt{vershynin2018high}, Theorem~4.6.1.) Thus $\|W^\top W\|_{\mathrm{op}} = \|W\|_{\mathrm{op}}^2 \to (1+\sqrt{\gamma})^2 = O(1)$.

For $W_T$: $\|W_T\|_{\mathrm{op}} \leq \max_i s_i \cdot \|\sign(W)\|_{\mathrm{op}}$. Since $s_i \to \sqrt{2/(\pi n)}$ a.s.\ and $\|\sign(W)\|_{\mathrm{op}} \leq \sqrt{m} + \sqrt{n} + t$ with probability $\geq 1 - 2e^{-t^2/2}$ (standard sub-Gaussian matrix concentration for $\pm 1$ entries; see e.g.\ \citealt{vershynin2018high}, Theorem~4.4.5), $\|W_T\|_{\mathrm{op}} = O(1)$ a.s.

Applying L\'{e}vy concentration \citep{ledoux2001concentration} to each of $\xhat^\top W^\top W\xhat$, $\xhat^\top W_T^\top W_T\xhat$, $\xhat^\top W^\top W_T\xhat$: each concentrates around its mean with deviation $O(n^{-1/2+\epsilon})$ for any $\epsilon > 0$.

\subsection{Triangular Array Chebyshev}

The row scales $s_i^2 = (\|w_i\|_1/n)^2$ form a triangular array (distribution of $w_i$ depends on~$n$). For $w_{ij} \sim N(0, 1/n)$:
\begin{itemize}
\item $\E[|w_{ij}|] = \sqrt{2/(\pi n)}$
\item $\Var(|w_{ij}|) = (1-2/\pi)/n$
\item $\E[s_i^2] = \E[\|w_i\|_1^2]/n^2 = (\Var(\|w_i\|_1) + (\E[\|w_i\|_1])^2)/n^2$
\item $= ((1-2/\pi) + 2n/\pi)/n^2 = 2/(\pi n) + (1-2/\pi)/n^2$
\item $\Var(s_i^2) = O(1/n^3)$ (from fourth-moment computation)
\end{itemize}

By Chebyshev: $\Var(\frac{1}{m}\sum_{i=1}^m s_i^2) = \Var(s_i^2)/m = O(1/(mn^3))$, so:
\[
\frac{1}{m}\sum_{i=1}^m s_i^2 \xrightarrow{P} \frac{2}{\pi n}
\]

\subsection{Cosine-Squared Computation Chain}

From L\'{e}vy concentration (combining with $W$-side concentration via SLLN):

\begin{enumerate}
\item $\|y\|^2 = \xhat^\top(W^\top W)\xhat \xrightarrow{P} \tr(W^\top W)/n = \|W\|_F^2/n \xrightarrow{P} m/n$ (using $\E[w_{ij}^2] = 1/n$ and concentration of $\sum_{ij}w_{ij}^2$ around~$m$).

\item $\|y_T\|^2 = \xhat^\top(W_T^\top W_T)\xhat \xrightarrow{P} \tr(W_T^\top W_T)/n$. Since row~$i$ of $W_T$ has squared norm $ns_i^2$, we have $\tr(W_T^\top W_T) = n\sum_i s_i^2$. By B.2: $\sum_i s_i^2 \xrightarrow{P} 2m/(\pi n)$. Thus $\|y_T\|^2 \xrightarrow{P} \sum_i s_i^2 \xrightarrow{P} 2m/(\pi n)$.

\item $y^\top y_T = \xhat^\top(W^\top W_T)\xhat \xrightarrow{P} \tr(W^\top W_T)/n$. Computing: $\tr(W^\top W_T) = \sum_i w_i^\top(s_i\sign(w_i)) = \sum_i s_i\|w_i\|_1 = n\sum_i s_i^2$ (using $s_i = \|w_i\|_1/n$). Thus $y^\top y_T \xrightarrow{P} \sum_i s_i^2 \xrightarrow{P} 2m/(\pi n)$.
\end{enumerate}

Combining:
\begin{equation}
\cos^2\!\angle(y, y_T) = \frac{(y^\top y_T)^2}{\|y\|^2\|y_T\|^2} \to \frac{(2m/(\pi n))^2}{(m/n)(2m/(\pi n))} = \frac{2m/(\pi n)}{m/n} = \frac{2}{\pi} \qquad \square
\end{equation}

\section{Proposition~2, Step 4 Detail}\label{app:prop2}

\textbf{Goal}: Upgrade the a.s.\ convergence $nT_n/S_n^2 \to \mu_4/\mu_2^2$ to $L^1$ convergence (i.e., $n\E[T_n/S_n^2] \to \mu_4/\mu_2^2$).

\textbf{Setup}: $T_n = \sum_{j=1}^n w_j^4$, $S_n = \sum_{j=1}^n w_j^2$, with $w_j$ i.i.d., $\mu_2 = \E[w_1^2] > 0$, $\mu_4 = \E[w_1^4] < \infty$. We additionally assume $\E[w_1^8] < \infty$ (satisfied by Gaussian and all sub-Gaussian distributions).

\subsection{Good/Bad Event Splitting}

Define the \emph{good event} $G_n := \{S_n \geq n\mu_2/2\}$.

\textbf{On $G_n$}: $nT_n/S_n^2 \leq nT_n/(n\mu_2/2)^2 = 4T_n/(n\mu_2^2)$. Since $T_n/n \to \mu_4$ a.s., it remains to justify $L^1$ convergence. The sequence $\{4T_n/(n\mu_2^2)\}$ is uniformly integrable: $\E[(T_n/n)^2] = \E[(\frac{1}{n}\sum_j w_j^4)^2] = \mu_4^2 + \text{Var}(w_1^4)/n \leq \mu_4^2 + \E[w_1^8]/n < \infty$ uniformly (using $\E[w_1^8] < \infty$). By the Vitali convergence theorem (a.s.\ convergence + UI $\Rightarrow$ $L^1$ convergence):
\[
\E\!\left[\frac{nT_n}{S_n^2}\mathbf{1}_{G_n}\right] \to \frac{\mu_4}{\mu_2^2}
\]

\subsection{Chernoff Bound on Bad Event}

\textbf{On $G_n^c$}: By $T_n \leq S_n^2$ (since $\sum a_j^2 \leq (\sum a_j)^2$ for $a_j \geq 0$): $nT_n/S_n^2 \leq n$.

The probability: for i.i.d.\ non-negative $w_j^2$ with mean $\mu_2$, by Chernoff/Hoeffding:
\[
\Pr(S_n/n < \mu_2/2) \leq \exp(-cn)
\]
for some $c > 0$ depending on the distribution of $w_1^2$.

Therefore: $\E[nT_n/S_n^2 \cdot \mathbf{1}_{G_n^c}] \leq n \cdot e^{-cn} \to 0$ exponentially fast.

\textbf{Combining}: $n\E[T_n/S_n^2] = \E[nT_n/S_n^2 \cdot \mathbf{1}_{G_n}] + \E[nT_n/S_n^2 \cdot \mathbf{1}_{G_n^c}] \to \mu_4/\mu_2^2 + 0$. \qed

\section{Additional Experimental Data}\label{app:data}

\subsection{Multi-Layer Detailed Results (V2, \texorpdfstring{$p{=}0.01$}{p=0.01}, 20 seeds)}

Table~\ref{tab:v2-detail} shows per-layer $C_L$ values for the V2 architecture (i.i.d.\ weights, no residual connection). The monotone decline contradicts exponential compounding.

\begin{table}[H]
\centering
\caption{V2 architecture per-layer results. $C_L$ shows monotone decline, contradicting the exponential compounding hypothesis.}
\label{tab:v2-detail}
\begin{tabular}{@{}cccc@{}}
\toprule
Layers $L$ & $C_L$ mean & 95\% CI & Trend \\
\midrule
1 & 3.587 & $[3.582,\, 3.593]$ & --- \\
2 & 3.494 & $[3.472,\, 3.516]$ & $-2.6\%$ \\
3 & 3.386 & $[3.339,\, 3.433]$ & $-5.6\%$ \\
4 & 3.278 & $[3.231,\, 3.324]$ & $-8.6\%$ \\
5 & 3.182 & $[3.135,\, 3.229]$ & $-11.3\%$ \\
6 & 3.066 & $[2.991,\, 3.140]$ & $-14.5\%$ \\
\bottomrule
\end{tabular}
\end{table}

\subsection{Residual Connection Effect (V3, \texorpdfstring{$p{=}0.01$}{p=0.01}, 20 seeds)}

Adding a residual connection $x + \text{FFN}(\xhat)$ substantially dilutes the asymmetry. The unperturbed skip path masks both sign-flip and magnitude perturbations, yielding a near-flat $C_L$ trajectory (Table~\ref{tab:v3-detail}).

\begin{table}[H]
\centering
\caption{V3 architecture (with residual connection) per-layer results. The near-flat trajectory confirms that residual connections dilute the sign/magnitude differential.}
\label{tab:v3-detail}
\begin{tabular}{@{}cccc@{}}
\toprule
Layers $L$ & $C_L$ mean & 95\% CI & Trend \\
\midrule
1 & 1.355 & $[1.351,\, 1.359]$ & --- \\
2 & 1.349 & $[1.345,\, 1.354]$ & $-0.4\%$ \\
3 & 1.343 & $[1.336,\, 1.349]$ & $-0.9\%$ \\
4 & 1.335 & $[1.327,\, 1.343]$ & $-1.5\%$ \\
5 & 1.331 & $[1.323,\, 1.339]$ & $-1.8\%$ \\
6 & 1.318 & $[1.307,\, 1.330]$ & $-2.7\%$ \\
\bottomrule
\end{tabular}
\end{table}

\subsection{Outlier Experiment Per-Seed Results}

Table~\ref{tab:outlier-seeds} reports per-seed breakdowns for the outlier sign-flip experiment (Exp~D, 20 seeds). The all-column $\Delta$PPL ratio exhibits high variance (IQR: $[795, 1{,}629]$), driven by the stochastic nature of which weight-outlier interactions are flipped. The count-matched leverage (rightmost column) controls for the 19:1 flip-count confound.

\begin{table}[H]
\centering
\caption{Outlier sign-flip experiment per-seed results (20 seeds). Median $\Delta$PPL ratio: $\mathbf{1{,}265\times}$; median count-matched leverage: $\mathbf{77\times}$.}
\label{tab:outlier-seeds}
\begin{tabular}{@{}cccccc@{}}
\toprule
Seed & Random PPL & Outlier PPL & Non-out.\ PPL & Out./Non-out. & CM leverage \\
\midrule
42 & 671 & 227{,}883 & 39.4 & $7{,}128\times$ & $96\times$ \\
123 & 425 & 33{,}850 & 42.9 & $954\times$ & $59\times$ \\
456 & 737 & 38{,}917 & 41.3 & $1{,}148\times$ & $83\times$ \\
789 & 1{,}162 & 21{,}004 & 37.7 & $694\times$ & $98\times$ \\
1000 & 540 & 31{,}423 & 39.6 & $974\times$ & $48\times$ \\
1111 & 952 & 48{,}086 & 40.6 & $1{,}449\times$ & $72\times$ \\
1234 & 2{,}557 & 46{,}722 & 36.5 & $1{,}606\times$ & $96\times$ \\
2024 & 1{,}006 & 100{,}428 & 42.2 & $2{,}885\times$ & $313\times$ \\
2222 & 1{,}698 & 39{,}715 & 42.8 & $1{,}122\times$ & $42\times$ \\
3141 & 1{,}487 & 17{,}469 & 40.9 & $520\times$ & $48\times$ \\
3333 & 2{,}452 & 109{,}661 & 43.5 & $3{,}036\times$ & $59\times$ \\
4444 & 1{,}532 & 180{,}549 & 42.1 & $5{,}198\times$ & $82\times$ \\
4567 & 1{,}368 & 26{,}240 & 39.6 & $813\times$ & $48\times$ \\
5555 & 786 & 52{,}136 & 41.7 & $1{,}518\times$ & $318\times$ \\
5678 & 916 & 25{,}692 & 42.2 & $739\times$ & $140\times$ \\
6666 & 902 & 56{,}379 & 40.6 & $1{,}698\times$ & $82\times$ \\
6789 & 889 & 43{,}739 & 36.7 & $1{,}494\times$ & $125\times$ \\
7890 & 1{,}631 & 16{,}238 & 42.3 & $465\times$ & $42\times$ \\
8901 & 1{,}596 & 44{,}323 & 39.4 & $1{,}382\times$ & $60\times$ \\
9012 & 771 & 24{,}924 & 42.0 & $719\times$ & $47\times$ \\
\midrule
\textit{Median} & \textit{979} & \textit{41{,}727} & \textit{41.1} & $\mathbf{1{,}265\times}$ & $\mathbf{77\times}$ \\
\bottomrule
\end{tabular}
\end{table}

\subsection{Radial Fraction Measurements (V1 Architecture)}

We measure the radial fractions $\mathcal{R}^{\mathrm{sign}}$ and $\mathcal{R}^{\mathrm{mag}}$ across dimensions $n \in \{256, 512, 1024, 2048\}$ and flip probabilities $p \in \{0.01, 0.02, 0.05, 0.10\}$, confirming the theoretical predictions of Theorem~3.

\begin{table}[H]
\centering
\caption{Measured sign-perturbation radial fractions $\mathcal{R}^{\mathrm{sign}}$. Theory predicts $\mathcal{R}^{\mathrm{sign}} \approx p$; convergence improves with dimension.}
\label{tab:radial-sign}
\begin{tabular}{@{}ccccc@{}}
\toprule
Dim & $p{=}0.01$ & $p{=}0.02$ & $p{=}0.05$ & $p{=}0.10$ \\
\midrule
256 & 0.0217 & 0.0308 & 0.0565 & 0.0969 \\
512 & 0.0156 & 0.0248 & 0.0509 & 0.0917 \\
1024 & 0.0125 & 0.0218 & 0.0481 & 0.0891 \\
2048 & 0.0110 & 0.0202 & 0.0465 & 0.0876 \\
\bottomrule
\end{tabular}
\end{table}

\begin{table}[H]
\centering
\caption{Measured magnitude-perturbation radial fractions $\mathcal{R}^{\mathrm{mag}}$. Theory predicts convergence to $2/\pi \approx 0.637$ as dimension grows.}
\label{tab:radial-mag}
\begin{tabular}{@{}ccccc@{}}
\toprule
Dim & $p{=}0.01$ & $p{=}0.02$ & $p{=}0.05$ & $p{=}0.10$ \\
\midrule
256 & 0.6451 & 0.6470 & 0.6505 & 0.6538 \\
512 & 0.6445 & 0.6465 & 0.6501 & 0.6534 \\
1024 & 0.6428 & 0.6448 & 0.6484 & 0.6518 \\
2048 & 0.6428 & 0.6449 & 0.6484 & 0.6518 \\
\bottomrule
\end{tabular}
\end{table}

\section{Outlier Amplification Theory}\label{app:outlier}

When Assumption~(A2) (delocalization) is violated---i.e., an input dimension $j^*$ carries macroscopic energy $|\hat x_{j^*}| = \alpha = \Theta(1)$---the gate-flip mechanism amplifies sign perturbations far beyond the $\pi/(\pi-2)$ baseline.  This appendix derives closed-form leverage formulas for a single outlier weight flip.

\subsection{Setup}

Consider the single-layer model of Section~5 with $W_1 \in \mathbb{R}^{n\times n}$, entries i.i.d.\ $\mathcal{N}(0, 1/n)$ (Assumption~A1).  The input $\hat x \in S^{n-1}$ has one outlier dimension:
\[
|\hat x_{j^*}| = \alpha = \Theta(1), \qquad |\hat x_j| = O(1/\sqrt{n}) \text{ for } j \neq j^*.
\]
We flip one weight: $W_{1,ij^*} \to -W_{1,ij^*}$, producing perturbed pre-activation
\[
z_i' = z_i - 2W_{1,ij^*}\hat x_{j^*}, \qquad \Delta z_i = -2W_{1,ij^*}\hat x_{j^*}.
\]

\subsection{Correlation Structure}

\begin{proposition}\label{prop:outlier-corr}
Under the single-entry flip model, $\operatorname{Corr}(z_i, z_i') = 1 - 2\alpha^2$.
\end{proposition}

\begin{proof}
Decompose $z_i = A + B$ where $A = W_{1,ij^*}\hat x_{j^*} \sim \mathcal{N}(0, \alpha^2/n)$ and $B = \sum_{j\neq j^*} W_{1,ij}\hat x_j \sim \mathcal{N}(0, (1-\alpha^2)/n)$, with $A \perp B$.  After the flip, $z_i' = -A + B$.  Then:
\begin{align*}
\operatorname{Var}(z_i) &= \operatorname{Var}(z_i') = \tfrac{1}{n}, \\
\operatorname{Cov}(z_i, z_i') &= \operatorname{Cov}(A+B,\, -A+B) = -\operatorname{Var}(A) + \operatorname{Var}(B) = \tfrac{1-2\alpha^2}{n}.
\end{align*}
Hence $\rho = (1-2\alpha^2)/n \div (1/n) = 1 - 2\alpha^2$.
\end{proof}

\begin{remark}
Since $A$ and $B$ are jointly Gaussian (both linear in i.i.d.\ Gaussian weights), the pair $(z_i, z_i')$ is \emph{exactly} bivariate normal.  This is not an asymptotic statement---it holds for all finite $n$ under~(A1).
\end{remark}

\subsection{Gate-Flip Probability}

\begin{theorem}[Outlier gate-flip probability]\label{thm:outlier-pflip}
For a single outlier weight flip with $|\hat x_{j^*}| = \alpha$:
\[
P_{\mathrm{flip}}(\alpha) = \frac{1}{\pi}\arccos(1 - 2\alpha^2) = \frac{2}{\pi}\arcsin(\alpha).
\]
\end{theorem}

\begin{proof}
By Proposition~\ref{prop:outlier-corr} and the preceding remark, $(z_i, z_i')$ is bivariate normal with correlation $\rho = 1 - 2\alpha^2$.  The Van~Vleck arcsine law gives $\Pr(\operatorname{sign}(z_i) \neq \operatorname{sign}(z_i')) = \frac{1}{\pi}\arccos\rho$.  The identity $\arccos(1-2\alpha^2) = 2\arcsin\alpha$ follows from $\cos(2\theta) = 1 - 2\sin^2\theta$ with $\theta = \arcsin\alpha$.
\end{proof}

\paragraph{Asymptotics.}
\begin{itemize}
\item $\alpha \to 0$: $P_{\mathrm{flip}} \approx 2\alpha/\pi$.  In the delocalized limit $\alpha = 1/\sqrt{n}$ (single entry), this gives $P_{\mathrm{flip}} \sim 2/(\pi\sqrt{n}) \to 0$, recovering the single-entry baseline; the $O(\sqrt{p})$ gate-flip rate of Theorem~\ref{thm:core} arises from aggregating $\mathrm{Bernoulli}(p)$ flips across all $n$ entries (Step~A7).
\item $\alpha = 1/\sqrt{2}$: $\rho = 0$, $P_{\mathrm{flip}} = 1/2$ (complete decorrelation).
\item $\alpha \to 1$: $P_{\mathrm{flip}} \to 1$ (deterministic gate flip).
\end{itemize}

\subsection{Gate-Flip Energy}

\begin{proposition}[Per-neuron gate-flip energy]\label{prop:outlier-betaflip}
The gate-flip energy contribution per neuron (physical scale) is:
\[
\beta_{\mathrm{flip}}^{\mathrm{outlier}}(\alpha) = \frac{1}{2\pi n}\bigl[\arccos(1-2\alpha^2) - (1-2\alpha^2)\cdot 2\alpha\sqrt{1-\alpha^2}\,\bigr].
\]
\end{proposition}

\begin{proof}
For standardized bivariate normal $(X, Y)$ with correlation $\rho$, the Plackett integral (cf.\ Step~A4 in Appendix~\ref{app:thm3}) gives:
\[
\mathbb{E}[X^2 \mathbf{1}_{X>0,\, Y<0}] = \frac{1}{4} - \frac{1}{2\pi}\bigl[\arcsin\rho + \rho\sqrt{1-\rho^2}\,\bigr].
\]
Substituting $\rho = 1-2\alpha^2$:
\begin{align*}
\arcsin(1-2\alpha^2) &= \tfrac{\pi}{2} - \arccos(1-2\alpha^2), \\
\rho\sqrt{1-\rho^2} &= (1-2\alpha^2)\cdot 2\alpha\sqrt{1-\alpha^2}.
\end{align*}
Hence:
\[
\mathbb{E}[X^2 \mathbf{1}_{X>0,Y<0}] = \frac{1}{2\pi}\bigl[\arccos(1-2\alpha^2) - (1-2\alpha^2)\cdot 2\alpha\sqrt{1-\alpha^2}\,\bigr].
\]
Scaling by $\operatorname{Var}(z_i) = 1/n$ yields the physical-scale result.
\end{proof}

\paragraph{Small-$\alpha$ expansion.}  Using $\arccos(1-2\alpha^2) = 2\arcsin\alpha \approx 2\alpha + \alpha^3/3$ and $(1-2\alpha^2)\cdot 2\alpha\sqrt{1-\alpha^2} \approx 2\alpha - 5\alpha^3$:
\[
\beta_{\mathrm{flip}}^{\mathrm{outlier}} \approx \frac{1}{2\pi n}\cdot\frac{16\alpha^3}{3} = \frac{8\alpha^3}{3\pi n}.
\]
Setting $p = \alpha^2$ formally recovers the Step~A4 result $8p^{3/2}/(3\pi)$ on unit-variance scale.

\subsection{Per-Flip Leverage Ratio}

\begin{theorem}[Outlier leverage ratio]\label{thm:outlier-leverage}
Define $R(\alpha, n) = E_{\mathrm{outlier}}/E_{\mathrm{non\text{-}outlier}}$ as the ratio of post-ReLU energy change from one outlier flip to one non-outlier flip.  Then:
\[
R(\alpha, n) = \frac{f(1-2\alpha^2)}{f(1-2/n)},
\]
where $f(\rho) = 1 - \rho + \frac{1}{\pi}\bigl(\rho\arccos\rho - \sqrt{1-\rho^2}\bigr)$ is the post-ReLU energy function for bivariate normal with correlation $\rho$.
\end{theorem}

\begin{proof}
For standard bivariate normal $(X,Y)$ with correlation $\rho$, the ReLU energy difference is:
\[
f(\rho) \;=\; \mathbb{E}[(\mathrm{ReLU}(X) - \mathrm{ReLU}(Y))^2] \;=\; 1 - \frac{1}{\pi}\bigl(\sin\theta + (\pi - \theta)\cos\theta\bigr),
\]
where $\theta = \arccos\rho$, which simplifies to the stated form.  Since both $z_i$ and $z_i'$ have variance $1/n$, the physical energy is $(1/n)\cdot f(\rho)$.  The ratio cancels the $1/n$ factor.  For a non-outlier entry with $|\hat x_j|^2 = 1/n$, the correlation is $\rho_{\mathrm{non}} = 1 - 2/n$.
\end{proof}

\paragraph{Leading-order expansion.}  Expanding $f(1-2\alpha^2)$ for small $\alpha$:
\[
f(1-2\alpha^2) = 2\alpha^2\Bigl(1 - \frac{4\alpha}{3\pi} + O(\alpha^2)\Bigr).
\]
For the non-outlier denominator, $f(1-2/n) = \frac{2}{n}\bigl(1 - \frac{4}{3\pi\sqrt{n}} + O(1/n)\bigr)$.  In the regime $\alpha \gg 1/\sqrt{n}$ (i.e., the outlier amplitude significantly exceeds the isotropic baseline $\alpha_0 = 1/\sqrt{n}$), the $O(1/\sqrt{n})$ denominator correction is negligible relative to the $O(\alpha)$ numerator correction, giving:
\begin{equation}\label{eq:leverage-ratio}
\boxed{\;R(\alpha, n) \;=\; n\alpha^2\Bigl(1 - \frac{4\alpha}{3\pi}\Bigr) + O(n\alpha^4) + O(\alpha^2/\!\sqrt{n}\,).\;}
\end{equation}
At leading order, $R \approx n\alpha^2$.  The correction $-4\alpha/(3\pi)$ arises from the nonlinear ReLU gate-flip interaction (the pre-ReLU energy ratio is exactly $n\alpha^2$ with no correction).  For $\alpha$ near $\alpha_c = 1/\sqrt{n}$ (the crossover where $R \approx 1$), the $O(\alpha^2/\sqrt{n})$ term contributes ${\sim}1\%$ at $n = 2048$.

\paragraph{Consistency with Theorem~\ref{thm:core}.}  Main text Step~A6 gives total post-ReLU energy $E(p) = 2p(1 - 4\sqrt{p}/(3\pi) + O(p))$ on unit-variance scale.  Setting $p = \alpha^2$: $E_{\mathrm{outlier}} = 2\alpha^2(1 - 4\alpha/(3\pi))$, matching $f(1-2\alpha^2)$.

\subsection{Numerical Predictions}

Table~\ref{tab:outlier-leverage} gives predictions for $n = 2048$ (TinyLlama hidden dimension).

\begin{table}[H]
\centering
\caption{Predicted outlier leverage ratio $R(\alpha, 2048)$ from Theorem~\ref{thm:outlier-leverage}.}
\label{tab:outlier-leverage}
\begin{tabular}{@{}ccccc@{}}
\toprule
$\alpha$ & Energy $\alpha^2$ & $P_{\mathrm{flip}}$ & $R$ (exact) & $R$ via Eq.~\eqref{eq:leverage-ratio} \\
\midrule
0.05 & 0.25\% & 3.2\% & 5.0 & 5.0 \\
0.10 & 1.0\% & 6.4\% & 19.5 & 19.6 \\
0.20 & 4.0\% & 12.8\% & 73 & 75 \\
0.30 & 9.0\% & 19.5\% & 162 & 161 \\
0.50 & 25\% & 33.3\% & 404 & 404 \\
\bottomrule
\end{tabular}
\end{table}

\subsection{Phase Crossover: Perturbative to Non-Perturbative}

The crossover occurs at the crossover scale
\[
\alpha_c = \frac{1}{\sqrt{n}},
\]
where a single outlier flip has the same impact as a single non-outlier flip ($R(\alpha_c, n) \approx 1$).  For $n = 2048$, $\alpha_c \approx 0.022$.

\begin{itemize}
\item \textbf{Delocalized regime} ($\alpha \ll \alpha_c$): Single-entry flip perturbs $z_i$ by $O(1/n) \ll \sigma_z = 1/\sqrt{n}$.  Gate-flip probability $\sim 2\alpha/\pi \to 0$.  This is the perturbative setting of Theorem~\ref{thm:core}.

\item \textbf{Outlier regime} ($\alpha \gg \alpha_c$): Single-entry flip perturbs $z_i$ by $O(\alpha/\sqrt{n}) \sim \sigma_z$.  Gate-flip probability $\sim (2/\pi)\arcsin\alpha = O(1)$.  Gate-flip energy dominates smooth energy.
\end{itemize}

The transition is a smooth crossover (no critical exponent or discontinuity): $R = n\alpha^2$ grows continuously from $R = 1$ at $\alpha = \alpha_c$ to $R \sim n$ at $\alpha \sim 1$.

\subsection{Connection to Experiments}

\paragraph{Theory vs.\ observation.}  Theorem~\ref{thm:outlier-leverage} predicts per-entry $R \in [20, 162]$ for extreme outlier amplitudes ($\alpha \in [0.1, 0.3]$, comprising ${\sim}16\%$ of outlier dimensions at TinyLlama layer~12).  The aggregate prediction, $n \cdot \mathbb{E}_{\mathrm{emp}}[\alpha^2] \approx 7$--$10$, is directly confirmed by the noise-floor sweep (Exp~E, Table~\ref{tab:noise-floor}): at $p \leq 0.5\%$ (linear-response regime), count-matched NLL leverage stabilizes at ${\sim}10\times$.  The $37.6\times$ NLL leverage at $p = 5\%$ (20 seeds; IQR: $[30, 46]$) reflects an empirically observed ${\sim}4\times$ multi-flip interaction inflation.

\paragraph{Measured $\alpha$ distribution.}  Direct measurement of $|\xhat_k|$ at TinyLlama layer~12 ($n = 2048$, 20{,}480 tokens) reveals the top-5\% outlier dimensions have $\alpha$ median $0.024$, mean $0.047$, and range $[0.017, 0.26]$ (P90 $= 0.12$, P95 $= 0.13$).  Only ${\sim}16\%$ of outlier dimensions exceed $\alpha = 0.1$; the bulk have $\alpha \in [0.02, 0.06]$.  The isotropic baseline is $\alpha_0 = \sqrt{2/(\pi n)} = 0.018$.  The aggregate count-matched NLL leverage at linear response (${\sim}10\times$ at $p \leq 0.5\%$; Exp~E) equals $n \cdot \mathbb{E}_{\mathrm{emp}}[\alpha^2]$; the elevated $37.6\times$ NLL at $p = 5\%$ reflects multi-flip interaction inflation dominated by the heavy-tailed $\alpha^2$ weighting of extreme outlier dimensions.

\paragraph{All-column gap resolved by log-loss.}  The all-column $\Delta$PPL ratio ($1{,}265\times$ median) compares flipping \emph{all} outlier-column weights (${\sim}2.4$M) against \emph{all} non-outlier-column weights (${\sim}46$M).  The column-group energy theory predicts $R_{\mathrm{col}} \leq 19$ for $\gamma \leq 0.5$ (\S\ref{sec:two-pop-gap}).  Since $\mathrm{PPL} = \exp(\mathrm{NLL})$, comparing PPL ratios conflates energy-level effects with the exponential nonlinearity of the perplexity metric itself.  Converting to log-loss (NLL $= \ln\mathrm{PPL}$):
\begin{itemize}
\item \emph{All-column NLL ratio}: $5.0\times$ (median, IQR $[4.8, 5.4]$), which falls \emph{within} $R_{\mathrm{col}} \leq 19$.
\item \emph{Count-matched NLL ratio}: $37.6\times$ (median, IQR $[30, 46]$) at $p = 5\%$; the noise-floor sweep (Exp~E) shows this includes ${\sim}4\times$ multi-flip inflation---at $p \leq 0.5\%$, leverage stabilizes at ${\sim}10\times = n \cdot \mathbb{E}_{\mathrm{emp}}[\alpha^2]$.
\item \emph{PPL compression factor}: The $1{,}265\times$ PPL ratio is $251\times$ inflated relative to the $5.0\times$ NLL ratio, confirming that PPL nonlinearity---not missing physics---accounts for the apparent gap.
\end{itemize}

The single-layer theory captures the correct \emph{scaling law} ($R \propto n\alpha^2$), the fundamental mechanism (gate-flip amplification under delocalization violation), and---as confirmed by both count-matched and log-loss analyses---the correct \emph{magnitude} of outlier leverage.  Multi-layer propagation (O1b) remains open but is no longer needed to explain the all-column observation.

\paragraph{Additional validation (Exps~E--G).}  Three experiments in Section~\ref{sec:outlier-exp} further validate the $R \propto n\alpha^2$ scaling: (i)~a noise-floor sweep (Exp~E, Table~\ref{tab:noise-floor}) confirms that count-matched NLL leverage stabilizes at ${\sim}10\times$ for $p \leq 0.5\%$, consistent with the per-entry ratio, while the inflated $41\times$ at $p = 5\%$ is attributable to multi-flip interactions; (ii)~activation-level perturbation energy (Exp~F) yields Spearman$(\alpha^2, \|\Delta y\|^2) = 0.955$ with weight column norms nearly uniform ($1.04\times$), confirming that the energy gap is driven by $\alpha^2$, not weight structure; (iii)~column-flip $\Delta$NLL (Exp~G) yields Spearman$(\alpha^2, \Delta\text{NLL}) = 0.927$ with log-log slope $0.78$, indicating sub-quadratic attenuation by downstream nonlinear processing.  The sub-quadratic effective exponent ($\alpha^{1.57}$ vs.\ predicted $\alpha^2$) is consistent with multi-layer compression of large perturbations and does not contradict the single-layer $R \propto n\alpha^2$ law.

\section{Two-Population Outlier Model}\label{app:two-pop}

This appendix extends the single-entry leverage theory of Appendix~\ref{app:outlier} to a \emph{population-level} model, deriving column-group leverage formulas relevant to Exp~D (Section~\ref{sec:outlier-exp}).

\subsection{Setup}

Partition the $n$ input dimensions into $k$ outlier and $n-k$ non-outlier dimensions:
\begin{itemize}
\item \textbf{Outlier}: $|\hat x_j| = \alpha$ for $j \in \mathcal{O}$, $|\mathcal{O}| = k$
\item \textbf{Non-outlier}: $|\hat x_j| = \beta$ for $j \notin \mathcal{O}$
\end{itemize}
with unit-norm constraint $k\alpha^2 + (n-k)\beta^2 = 1$.

Define the key parameters:
\begin{align}
\eta &= k/n \quad \text{(outlier column fraction)}, \\
\gamma &= k\alpha^2 \quad \text{(outlier energy fraction)}.
\end{align}
For Exp~D: $\eta = 0.05$ (top 5\% columns).

\subsection{Per-Flip Leverage Under Uniform Flips}

From Theorem~\ref{thm:outlier-leverage}, flipping a single entry $W_{1,ij}$ in column~$j$ produces post-ReLU energy proportional to $\hat{x}_j^2$. Define the damage fraction from outlier flips under uniform $p$-flips:
\[
F_{\mathrm{outlier}} = \gamma.
\]
The fraction of total damage from outlier columns \emph{equals} the outlier energy fraction, regardless of the detailed activation structure. Per-flip, however, outlier columns have amplified impact: each outlier flip has leverage $n\alpha^2 = \gamma/\eta$ relative to a delocalized entry.

\subsection{Column-Group Leverage}\label{sec:column-group}

Exp~D does not perform uniform $p$-flips; it compares flipping \emph{all} outlier-column weights versus \emph{all} non-outlier-column weights. This is a collective perturbation whose analysis requires the bivariate-Gaussian post-ReLU energy function.

\paragraph{Setup.} For neuron~$i$, decompose the pre-activation as $z_i = A_i + B_i$ where:
\begin{align}
A_i &= \textstyle\sum_{j \in \mathcal{O}} W_{1,ij}\hat{x}_j \sim N(0, \gamma/n), \\
B_i &= \textstyle\sum_{j \notin \mathcal{O}} W_{1,ij}\hat{x}_j \sim N(0, (1{-}\gamma)/n),
\end{align}
with $A_i \perp B_i$ (independence by disjoint index sets).

\paragraph{Case A (outlier flip):} Flip all $k$ outlier-column weights, negating $A_i$. The perturbed pre-activation is $z_i' = -A_i + B_i$. The correlation between $z_i = A_i + B_i$ and $z_i' = -A_i + B_i$ is:
\[
\rho_A = \frac{\Var(B_i) - \Var(A_i)}{\Var(z_i)} = \frac{(1{-}\gamma) - \gamma}{1} = 1 - 2\gamma.
\]

\paragraph{Case B (non-outlier flip):} Flip all $(n{-}k)$ non-outlier-column weights, negating $B_i$:
\[
\rho_B = \frac{\Var(A_i) - \Var(B_i)}{\Var(z_i)} = 2\gamma - 1.
\]

Note $\rho_B = -\rho_A$: the two perturbations are ``complementary'' in correlation space.

\paragraph{Post-ReLU energy function.} For bivariate Gaussian $(X, Y)$ with equal variances $\sigma^2$ and correlation $\rho$:
\begin{equation}\label{eq:f-rho}
f(\rho) := \frac{\E[\|\ReLU(X) - \ReLU(Y)\|^2]}{\sigma^2} = 1 - \rho + \frac{1}{\pi}\!\left(\rho\arccos\rho - \sqrt{1-\rho^2}\right).
\end{equation}
This follows from the Price--Cho--Saul formula. Key values:
\begin{itemize}
\item $f(1) = 0$ (perfect correlation: no perturbation)
\item $f(0) = 1 - 1/\pi \approx 0.682$
\item $f(-1) = 2$ (anti-correlation: maximum perturbation)
\end{itemize}

\paragraph{Column-group leverage ratio.} The per-flip leverage (accounting for flip-count asymmetry) is:
\begin{equation}\label{eq:R-col}
R_{\mathrm{col}} = \frac{1-\eta}{\eta} \cdot \frac{f(1-2\gamma)}{f(2\gamma-1)}.
\end{equation}

\begin{table}[!ht]
\centering
\caption{Column-group leverage $R_{\mathrm{col}}$ for $\eta = 0.05$ as a function of energy concentration~$\gamma$.}
\label{tab:two-pop}
\begin{tabular}{@{}ccccccc@{}}
\toprule
$\gamma$ & $\rho_A$ & $\rho_B$ & $f(\rho_A)$ & $f(\rho_B)$ & $f(\rho_A)/f(\rho_B)$ & $R_{\mathrm{col}}$ \\
\midrule
0.10 & $+0.80$ & $-0.80$ & 0.173 & 0.973 & 0.178 & 3.4 \\
0.20 & $+0.60$ & $-0.60$ & 0.322 & 0.922 & 0.350 & 6.6 \\
0.30 & $+0.40$ & $-0.40$ & 0.456 & 0.856 & 0.533 & 10.1 \\
0.40 & $+0.20$ & $-0.20$ & 0.575 & 0.775 & 0.742 & 14.1 \\
0.50 & $0.00$ & $0.00$ & 0.682 & 0.682 & 1.000 & \textbf{19.0} \\
\midrule
0.70 & $-0.40$ & $+0.40$ & 0.856 & 0.456 & 1.878 & 35.7 \\
0.90 & $-0.80$ & $+0.80$ & 0.973 & 0.173 & 5.626 & 106.9 \\
0.99 & $-0.98$ & $+0.98$ & 0.999 & 0.019 & 52.6 & ${\sim}1000$ \\
\bottomrule
\end{tabular}
\end{table}

\paragraph{Key observation.}\label{sec:two-pop-gap} For $\gamma \leq 0.5$, $R_{\mathrm{col}} \leq (1-\eta)/\eta = 19$ (since $\rho_A \geq 0$ and $\rho_B \leq 0$ imply $f(\rho_A) \leq f(\rho_B)$). Achieving $R_{\mathrm{col}} = 1{,}265$ (the observed all-column $\Delta$PPL ratio) from energy alone requires $\gamma \approx 0.99$---completely unrealistic for typical activations.

\subsection{Connection to Count-Matched Experiment}

The count-matched per-flip leverage ($77\times$ PPL, $37.6\times$ NLL at $p = 5\%$) isolates per-flip sensitivity from the count confound. The noise-floor sweep (Exp~E) reveals that this $p = 5\%$ value includes ${\sim}4\times$ multi-flip interaction inflation: at $p \leq 0.5\%$ (linear-response regime), the NLL leverage stabilizes at ${\sim}10\times$. This matches the aggregate prediction $n \cdot \mathbb{E}_{\mathrm{emp}}[\alpha^2] \approx 7$--$10$ from the measured outlier $\alpha$ distribution (median $0.024$, P90 $= 0.12$, max $= 0.26$; most top-5\% dimensions have $\alpha < 0.1$, but the aggregate is dominated by heavy-tailed $\alpha^2$ weighting of extreme outliers).

The apparent all-column gap ($1{,}265\times$ PPL$/19 \approx 67\times$) is accounted for by log-loss analysis: the all-column NLL ratio is only $5.0\times$ (median, IQR $[4.8, 5.4]$), well within $R_{\mathrm{col}} \leq 19$. The $251\times$ compression from PPL to NLL ratio indicates that the gap is predominantly an artifact of the exponential PPL metric rather than a failure of the energy-level theory.

\subsection{Connection to Dettmers et al.}

\citet{dettmers2022} reported that ${\sim}0.1\%$ of hidden-state dimensions carry magnitudes $50$--$100\times$ larger than average. In our notation: $\eta \approx 0.001$, $\alpha/\beta \approx 50$--$100$.

The per-entry leverage is \emph{not} simply $(\alpha/\beta)^2$; the unit-norm constraint requires:
\[
R_{\mathrm{entry}} = \frac{\gamma}{\eta} = \left(\frac{\alpha}{\beta}\right)^{\!2} \cdot (1-\gamma).
\]
For Dettmers-scale outliers ($\gamma \approx 0.7$--$0.9$), this gives $R_{\mathrm{entry}} \approx 700$--$900$ (not $2{,}500$--$10{,}000$ as a na\"ive calculation would suggest). This is consistent with Dettmers' observation that INT8 quantization fails catastrophically when outlier features exceed the representable range.

\section{Post-ReLU Ternary Error Derivation}\label{app:ternary-post}

This appendix provides the complete derivation of the post-ReLU radial fraction $\mathcal{R}^{\mathrm{ternary}}_{\mathrm{post}} \approx 0.365$ stated in Section~\ref{sec:post-relu-ternary}.

\subsection{Pre-Activation Statistics}

Let $z_i = \sum_j W_{1,ij}\hat{x}_j$ (original) and $z_i^Q = s_i \sum_j \sign(W_{1,ij})\hat{x}_j$ (quantized), where $s_i = \|w_i\|_1/n \to \sqrt{2/(\pi n)}$ by SLLN.

\paragraph{Marginal distributions.}
\[
z_i \sim N(0, \sigma_1^2), \quad \sigma_1^2 = 1/n; \qquad z_i^Q \text{ has variance } \sigma_2^2 = s_i^2 = 2/(\pi n).
\]
Note $\sigma_2 = \sqrt{2/\pi}\,\sigma_1 \approx 0.798\,\sigma_1$: the quantized model produces smaller pre-activations.

\paragraph{Joint distribution.} By the multivariate Lindeberg--Feller CLT under delocalization (A2):
\[
(z_i, z_i^Q) \xrightarrow{d} N\!\left(0, \begin{pmatrix} 1/n & 2/(\pi n) \\ 2/(\pi n) & 2/(\pi n) \end{pmatrix}\right).
\]
The covariance $\Cov(z_i, z_i^Q) = s_i \cdot \E[|W_{1,ij}|] = 2/(\pi n)$, giving:
\begin{equation}\label{eq:rho-ternary}
\rho = \frac{2/(\pi n)}{\sqrt{(1/n)\cdot 2/(\pi n)}} = \sqrt{2/\pi} \approx 0.798.
\end{equation}

\subsection{Post-ReLU Energy via Price--Cho--Saul}

Define $\phi = \arccos(\sqrt{2/\pi}) \approx 0.647\,\mathrm{rad}$.

\paragraph{Cross term.} For bivariate normals with unequal variances:
\[
\E[\ReLU(z_i)\,\ReLU(z_i^Q)] = \frac{\sigma_1\sigma_2}{2\pi}\!\left(\sin\phi + (\pi-\phi)\cos\phi\right) = \frac{S}{n},
\]
where the auxiliary constant is:
\begin{equation}\label{eq:S-def}
S = \frac{\cos\phi\,(\sin\phi + (\pi-\phi)\cos\phi)}{2\pi} \approx 0.3293.
\end{equation}

\paragraph{Per-neuron error energy.}
\begin{align}
\E[\delta a_i^2] &= \E[\ReLU(z_i)^2] + \E[\ReLU(z_i^Q)^2] - 2\,\E[\ReLU(z_i)\,\ReLU(z_i^Q)] \nonumber \\
&= \frac{\sigma_1^2}{2} + \frac{\sigma_2^2}{2} - \frac{2S}{n} = \frac{1}{n}\!\left(\frac{1}{2} + \frac{1}{\pi} - 2S\right) \approx \frac{0.1597}{n}. \label{eq:per-neuron-energy}
\end{align}

\subsection{Radial Component}

\paragraph{Per-neuron radial contribution.}
\[
\E[a_i\,\delta a_i] = \E[\ReLU(z_i)\,\ReLU(z_i^Q)] - \E[\ReLU(z_i)^2] = \frac{S - \tfrac{1}{2}}{n} \approx \frac{-0.1707}{n}.
\]
The negative sign reflects radial \emph{contraction}: $\sigma_2 < \sigma_1$.

\paragraph{Radial projection.} By SLLN, $\|a\|^2 = \sum_{i=1}^{m} \ReLU(z_i)^2 \to m \cdot \sigma_1^2/2 = m/(2n)$, so $\|a\| \to \sqrt{m/(2n)}$. Thus:
\[
\langle \delta a, \hat{a}\rangle = \frac{\sum_i a_i\,\delta a_i}{\|a\|} \to \frac{m \cdot (S-\tfrac{1}{2})/n}{\sqrt{m/(2n)}} = (S - \tfrac{1}{2})\sqrt{2m/n}.
\]
Although the intermediate quantities depend on $m/n$, the ratio $m/n$ cancels in the final radial fraction (Section~\ref{sec:assembly-G}).

\subsection{Assembly}\label{sec:assembly-G}

The post-ReLU radial fraction:
\begin{equation}\label{eq:R-post-ternary}
\mathcal{R}^{\mathrm{ternary}}_{\mathrm{post}} = \frac{\langle\delta a, \hat{a}\rangle^2}{\|\delta a\|^2} = \frac{2(m/n)(S-\tfrac{1}{2})^2}{(m/n)(\tfrac{1}{2} + \tfrac{1}{\pi} - 2S)} = \frac{2(S-\tfrac{1}{2})^2}{\tfrac{1}{2} + \tfrac{1}{\pi} - 2S} \approx \frac{0.0583}{0.1597} \approx 0.365.
\end{equation}
The factor $m/n$ cancels between numerator and denominator, making the radial fraction independent of the aspect ratio $m/n$.

Compared to the pre-ReLU value $\mathcal{R}^{\mathrm{ternary}}_{\mathrm{pre}} = 1 - 2/\pi \approx 0.363$ (exact by Bussgang orthogonality), the difference is ${\sim}0.4\%$ relative. The pre-ReLU exactness follows from $\E[\delta z \cdot z^Q] = 0$: the quantization error is $L^2$-orthogonal to the quantized output, a structural consequence of choosing $s_i = \E[|W_{ij}|]$ as the per-row scale.

\subsection{Gate-Flip Analysis}

By the Van Vleck arcsine law:
\[
P_{\mathrm{flip}} = \frac{1}{\pi}\arccos\!\left(\sqrt{2/\pi}\right) = \frac{\phi}{\pi} \approx 20.6\%.
\]
This is scale-invariant (depends only on $\rho$, not on $\sigma_1, \sigma_2$ individually). Despite this substantial gate-flip rate, the post-ReLU radial fraction remains within $0.4\%$ of its pre-ReLU value, indicating that the gate-preserved and gate-flipped contributions to the radial fraction approximately cancel.

\paragraph{Structural conjecture.} Gate-preserved neurons (${\sim}79.4\%$) contribute a ``smooth'' radial fraction $\mathcal{R}_{\mathrm{smooth}} > 1-2/\pi$ (the perturbation is a magnitude rescaling), while gate-flipped neurons (${\sim}20.6\%$) contribute $\mathcal{R}_{\mathrm{flip}} < 1-2/\pi$ (flips inject purely transverse energy). The weighted average $0.794\,\mathcal{R}_{\mathrm{smooth}} + 0.206\,\mathcal{R}_{\mathrm{flip}} \approx 0.365$ produces the observed near-equality. Verifying this decomposition requires computing conditional expectations over the four quadrants of $(z_i, z_i^Q)$ and is left as an open question.

\end{document}